\documentclass{article}

\usepackage{microtype}
\usepackage{graphicx}
\usepackage{subfigure}
\usepackage{booktabs} % for professional tables
\usepackage{amsthm}
\usepackage{bm}
\usepackage{url}

\usepackage{hyperref}

\usepackage[accepted]{icml2020}

\usepackage{amsmath}
\usepackage{multirow}
\usepackage{extarrows}

\newtheorem{theorem}{Theorem}
\newtheorem{lemma}{Lemma}
\newtheorem{conjecture}{Conjecture}

\icmltitlerunning{Simple and Deep Graph Convolutional Networks}
\begin{document}

\twocolumn[
\icmltitle{Simple and Deep Graph Convolutional Networks}

\begin{icmlauthorlist}
\icmlauthor{Ming Chen}{rucinfo}
\icmlauthor{Zhewei Wei}{gaoling,bj,moe}
\icmlauthor{Zengfeng Huang}{fud}
\icmlauthor{Bolin Ding}{ali}
\icmlauthor{Yaliang Li}{ali}
\end{icmlauthorlist}

\icmlaffiliation{rucinfo}{School of Information, Renmin University of China}
\icmlaffiliation{gaoling}{Gaoling School of Articial Intelligence, Renmin University of China}
% \icmlaffiliation{bdma}{Beijing Key Lab of Big Data Management and Analysis Methods, Renmin University of China}
% \icmlaffiliation{deke}{MOE Key Lab of Data Engineering and Knowledge Engineering, Renmin University of China}
\icmlaffiliation{fud}{School of Data Science, Fudan University}
\icmlaffiliation{ali}{Alibaba Group}
\icmlaffiliation{bj}{Beijing Key Lab of Big Data Management and Analysis Methods}
\icmlaffiliation{moe}{MOE Key Lab of Data Engineering and Knowledge Engineering}

\icmlcorrespondingauthor{Zhewei Wei}{zhewei@ruc.edu.cn}

\icmlkeywords{Graph Convolutional Networks}

\vskip 0.3in
]



\printAffiliationsAndNotice{} % otherwise use the standard text.

\begin{abstract}
 Graph convolutional networks (GCNs) are a powerful deep learning
approach for graph-structured data.  Recently, GCNs and subsequent
variants have shown superior performance in various application areas on
real-world datasets. Despite their success, most of the current GCN models are
shallow, due to the {\em over-smoothing} problem. In this paper, we
study the problem of designing and analyzing deep graph convolutional networks.  We
propose the GCNII, an extension of the  vanilla GCN model with two
simple yet effective techniques: {\em Initial residual} and {\em Identity mapping}. We provide theoretical and
empirical evidence that the two techniques effectively relieves the problem of over-smoothing. 
Our experiments show that the deep GCNII model outperforms the
state-of-the-art methods on various semi- and full-supervised tasks. 
Code is available at \url{https://github.com/chennnM/GCNII}.
\end{abstract}

\section{Introduction}
\label{sec:intro}
Graph convolutional networks (GCNs)~\cite{DBLP:conf/iclr/KipfW17}  generalize
convolutional neural networks (CNNs)~\cite{lecun1995convolutional} to graph-structured data. To learn the graph
representations, the ``graph convolution'' operation applies the same
linear transformation to all the neighbors of a node followed by a
nonlinear activation function. In recent years, GCNs and their
variants~\cite{DBLP:conf/nips/DefferrardBV16, velickovic2018graph} 
have been successfully applied to a wide range of applications,
including
social analysis~\cite{DBLP:conf/kdd/QiuTMDW018,DBLP:conf/acl/LiG19},
traffic
prediction~\cite{DBLP:conf/aaai/GuoLFSW19,DBLP:conf/kdd/LiHCSWZP19},
biology~\cite{DBLP:conf/nips/FoutBSB17,DBLP:conf/aaai/ShangXMLS19},
recommender systems~\cite{DBLP:conf/kdd/YingHCEHL18}, and computer
vision~\cite{DBLP:conf/cvpr/00030TKM19,DBLP:journals/dase/MaWZCL19}. 

% including
% social analysis~\cite{DBLP:conf/kdd/QiuTMDW018,DBLP:conf/acl/LiG19},
% traffic
% prediction~\cite{DBLP:conf/aaai/GuoLFSW19,DBLP:conf/kdd/LiHCSWZP19},
% biology~\cite{DBLP:conf/nips/FoutBSB17,DBLP:conf/aaai/ShangXMLS19},
% combinatorial
% optimization~\cite{DBLP:conf/nips/LiCK18,DBLP:conf/aaai/PratesALLV19},
% recommender systems~\cite{DBLP:conf/kdd/YingHCEHL18}, reinforcement
% learning~\cite{DBLP:conf/naacl/AmmanabroluR19}, natural language
% processing~\cite{DBLP:conf/naacl/MarcheggianiBT18,DBLP:conf/aaai/YaoM019}
% and computer
% vision~\cite{DBLP:conf/cvpr/00030TKM19,DBLP:conf/cvpr/ChenWWG19}. 

Despite their enormous success, most of the current GCN models are
shallow. Most of the recent
models, such as GCN~\cite{DBLP:conf/iclr/KipfW17}  and
GAT~\cite{velickovic2018graph}, achieve their best performance with 
2-layer models. Such shallow architectures limit their ability to
extract information from
high-order neighbors. However, stacking more layers and adding non-linearity tends
to degrade the performance of these models.
Such a phenomenon is called {\em
  over-smoothing}~\cite{DBLP:conf/aaai/LiHW18}, which suggests that as
the number of layers increases, the representations of the nodes in GCN are inclined to
converge to a certain value and thus become
indistinguishable.
ResNet~\cite{DBLP:conf/cvpr/HeZRS16} solves a  similar problem in  computer
vision with {\em residual connections}, which is effective for training
very deep neural networks. Unfortunately, adding residual connections
in the GCN models merely slows down the over-smoothing
problem~\cite{DBLP:conf/iclr/KipfW17};
deep GCN models are still outperformed by 2-layer models such as GCN or GAT.

Recently, several works try to tackle the problem of
over-smoothing. JKNet~\cite{DBLP:conf/icml/XuLTSKJ18} uses dense skip connections
to combine the output of each layer to preserve the locality of the node
representations. 
Recently, DropEdge~\cite{rong2020dropedge}
suggests that by randomly removing out a few edges from the input graph,
one can relieve the impact of
over-smoothing. Experiments~\cite{rong2020dropedge} suggest that the two methods
can slow down the performance drop as we increase the network
depth. However, for semi-supervised tasks, the
state-of-the-art results are still achieved by the shallow models,
and thus  the benefit brought by increasing the network depth remains in doubt.
% both work
% focus on the full-supervised tasks; For semi-supervised tasks, 
% the
% state-of-the-art results are still achieved by shallow
% models. Moreover,

On the other hand, several methods combine deep propagation with shallow neural networks.
SGC~\cite{pmlr-v97-wu19e}  attempts to
capture higher-order information in the graph by applying the $K$-th
power of the graph
convolution matrix in a single neural network layer.
PPNP and APPNP~\cite{klicpera_predict_2019} replace the power of
the graph convolution matrix with the Personalized PageRank matrix to
solve the over-smoothing
problem. GDC~\cite{DBLP:conf/nips/KlicperaWG19} further extends APPNP
by generalizing Personalized PageRank~\cite{page1999pagerank}
to an arbitrary graph diffusion process. However, these methods perform a linear
combination of neighbor features in each layer and lose the powerful
expression ability of deep nonlinear architectures,  which means they
are still shallow models.

In conclusion, it remains an open problem to
design a GCN model that effectively
prevents over-smoothing and achieves state-of-the-art results
with  truly deep network structures. 
Due to this challenge, it is even unclear whether the network depth is a
resource or a burden in designing new graph neural networks. 
In this paper, we give a positive answer to this open problem by demonstrating that the vanilla GCN~\cite{DBLP:conf/iclr/KipfW17} 
can be extended to a deep model with two simple yet effective modifications. In
particular, we propose \textbf{G}raph \textbf{C}onvolutional \textbf{N}etwork via
\textbf{I}nitial residual and \textbf{I}dentity mapping ({\bf GCNII}), a deep GCN model that resolves the over-smoothing problem.
At each layer, initial
residual constructs a skip connection from the input
layer, while identity mapping adds an identity matrix
to the weight matrix.
The empirical study demonstrates that the two surprisingly simple
techniques prevent over-smoothing and improve the performance of GCNII
consistently as we increase its
network depth. In
particular, the deep GCNII model  achieves new state-of-the-art results on
various semi-supervised and full-supervised tasks.

Second, we provide theoretical analysis for multi-layer GCN and
GCNII models.  It
is known~\cite{pmlr-v97-wu19e}   that by stacking $k$ layers, the vanilla GCN essentially simulates a
$K$-th order of polynomial filter with predetermined
coefficients. \cite{wang2019improving} points out that such a filter
simulates a {\em lazy random walk}  that eventually
converges to the stationary vector and thus leads to
over-smoothing. 
On the other hand, we prove that  a $K$-layer GCNII model can express a
polynomial spectral filter of order $K$ with arbitrary
coefficients. This property is essential for designing deep neural
networks. We also derive the closed-form of the
stationary vector and analyze the rate of convergence for the vanilla
GCN. Our analysis implies that nodes with high degrees are
more likely to suffer from over-smoothing  in a multi-layer GCN model, and we perform experiments
to confirm this theoretical conjecture.

\section{Preliminaries}
\label{sec:pre}

\paragraph{Notations.}
Given a simple and connected undirected graph \( G=(V,E)\) with $n$ nodes and $m$ edges. We  define the {\em self-looped graph}  \( \tilde{G}=(V, \tilde{E})\) to be the graph with a self-loop attached to each node in $G$. 
We use $\{1, \ldots, n\}$ to denote the node IDs of $G$ and $ \tilde{G}$, and $d_j$ and $d_j+1$ to denote the degree of node $j$ in $G$ and $\tilde{G}$, respectively.  Let $\mathbf{A}$ denote the adjacency matrix and $\mathbf{D}$ the diagonal degree matrix. Consequently,
the adjacency matrix and diagonal degree matrix of $ \tilde{G}$  is defined to be \(\tilde{\mathbf{A}}=\mathbf{A}+\mathbf{I}\) and 
\(\tilde{\mathbf{D}}=\mathbf{D}+\mathbf{I}\), respectively. 
Let $\mathbf{X} \in \mathbf{R}^{n \times d}$ denote the node feature matrix, that is, each node $v$ is associated with a $d$-dimensional feature vector $\mathbf{X}_v$. The {\em normalized graph Laplacian matrix} is defined as  \(\mathbf{L}=\mathbf{I}_{n}-\) \(\mathbf{D}^{-1 / 2} \mathbf{A} \mathbf{D}^{-1 / 2}\), which is a symmetric positive semidefinite matrix with eigendecomposition \(\mathbf{U} \Lambda \mathbf{U}^{T},\). Here \(\Lambda \) is a diagonal matrix of the eigenvalues of $\mathbf{L}$, and  \(\mathbf{U} \in
\mathbf{R}^{n \times n}\) is a unitary matrix that consists of the eigenvectors of $\mathbf{L}$. 
The graph convolution operation between signal  \(\mathbf{x}\) and  filter \(g_{\gamma}(\Lambda)=\operatorname{diag}(\gamma)\) is defined as
\(g_{\gamma}(\mathbf{L}) * \mathbf{x}=\mathbf{U} g_{\gamma}(\Lambda) \mathbf{U}^{T} \mathbf{x},\) where the parameter \(\gamma \in \mathbf{R}^{n}\) corresponds to a vector of spectral filter coefficients. 

\paragraph{Vanilla GCN.}
\cite{DBLP:conf/iclr/KipfW17}  and~\cite{DBLP:conf/nips/DefferrardBV16} suggest that the graph convolution operation can be further approximated by the $K$-th order polynomial of Laplacians
$$
\mathbf{U} g_{\theta}(\Lambda) \mathbf{U}^{T} \mathbf{x} \approx \mathbf{U}\left(\sum_{\ell=0}^{K} \theta_{\ell} \mathbf{\Lambda}^{\ell}\right) \mathbf{U}^{\top} \mathbf{x} =\left(\sum_{\ell=0}^{K} \theta_{\ell} \mathbf{L}^{\ell}\right) \mathbf{x},
$$
where \(\theta \in \mathbf{R}^{K+1}\)  corresponds to a vector of polynomial coefficients. The vanilla GCN~\cite{DBLP:conf/iclr/KipfW17}  sets $K=1$, $\theta_0=2\theta$ and $\theta_1 = -\theta$ to obtain the convolution operation $
\mathbf{g_\theta} * \mathbf{x}=\theta\left(\mathbf{I}+\mathbf{D}^{-1 / 2} \mathbf{A} \mathbf{D}^{-1 / 2}\right) \mathbf{x}
$. Finally, by {\em the renormalization trick}, \cite{DBLP:conf/iclr/KipfW17} replaces the matrix \(\mathbf{I}+\mathbf{D}^{-1 / 2} \mathbf{A} \mathbf{D}^{-1 / 2}\) by a normalized version \( \tilde{\mathbf{P}} = \tilde{\mathbf{D}}^{-1 / 2} \tilde{\mathbf{A}} \tilde{\mathbf{D}}^{-1 / 2} = (\mathbf{D}+ \mathbf{I}_n)^{-1/ 2} (\mathbf{A}+ \mathbf{I}_n)(\mathbf{D}+ \mathbf{I}_n)^{-1/2}\). and obtains the Graph Convolutional Layer
\begin{equation}
    \label{eqn:gcn}
    \mathbf{H}^{(\ell+1)} = \sigma\left( \tilde{\mathbf{P}}  \mathbf{H}^{(\ell)}\mathbf{W}^{(\ell)}\right).
  \end{equation}
Where $\sigma$ denotes the ReLU operation.

SGC~\cite{pmlr-v97-wu19e} shows that by stacking $K$ layers, GCN corresponds to a {\em  fixed  polynomial filter} of order $K$ on the graph spectral domain of $\tilde{G}$. In particular, let \(\tilde{\mathbf{L}}=\mathbf{I}_{n}- \tilde{\mathbf{D}}^{-1 / 2} \tilde{\mathbf{A}} \tilde{\mathbf{D}}^{-1 / 2}\) denote 
the normalized graph Laplacian matrix  of the self-looped graph $\tilde{G}$.  Consequently, applying a $K$-layer GCN to a signal $\mathbf{x}$ corresponds to  \(\left(\tilde{\mathbf{D}}^{-1 / 2} \tilde{\mathbf{A}} \tilde{\mathbf{D}}^{-1 / 2} \right)^K\mathbf{x} = \left(\mathbf{I}_{n}- \tilde{\mathbf{L}}\right)^K \mathbf{x} \). \cite{pmlr-v97-wu19e}  also shows that by adding a self-loop to each node, $\tilde{\mathbf{L}}$ effectively shrinks the underlying graph spectrum.

\paragraph{APPNP.}
\cite{klicpera_predict_2019} uses Personalized PageRank to derive a fixed filter of order $K$. Let $f_{\theta}(\mathbf{X})$ denote the output of a two-layer fully connected neural network on the feature matrix $\mathbf{X}$, PPNP's model is defined as
\begin{equation}
    \label{eqn:ppnp}
    \mathbf{H} = \alpha\left(\mathbf{I}_{n}-(1-\alpha)\tilde{\mathbf{A}}\right)^{-1} f_{\theta}(\mathbf{X}).
  \end{equation}
  Due to the property of Personalized PageRank, such a filter preserves locality and thus is suitable for classification tasks. 
  \cite{klicpera_predict_2019}  also proposes APPNP, which replaces $\alpha\left(\mathbf{I}_{n}-(1-\alpha)\tilde{\mathbf{A}}\right)^{-1}$ with an approximation derived by a truncated power iteration. Formally, APPNP with $K$-hop aggregation is defined as
  \begin{equation}
\label{eqn:APPNP}
    \boldsymbol{H}^{(\ell+1)} =(1-\alpha) \tilde{\boldsymbol{P}} \boldsymbol{H}^{(\ell)}+\alpha \boldsymbol{H}^{(0)},
  \end{equation}
where $ \boldsymbol{H}^{(0)} =f_{\theta}(\boldsymbol{X})$.
  By decoupling feature transformation and propagation, PPNP and APPNP can aggregate information from multi-hop neighbors without increasing the number of layers in the neural network. 

\paragraph{JKNet.}
The first  deep GCN framework is proposed by \cite{DBLP:conf/icml/XuLTSKJ18}.  At the last layer, JKNet combines all previous representations $\left[\mathbf{H}^{(1)}, \ldots, \mathbf{H}^{(K)}\right]$ to learn representations of different orders for different graph substructures. 
\cite{DBLP:conf/icml/XuLTSKJ18}  proves that 1) a $K$-layer vanilla GCN model simulates random walks of $K$ steps in the self-looped graph $\tilde{G}$ and 2) by combining all representations from the previous layers, JKNet relieves the problem of over-smoothing. 

\paragraph{DropEdge}
A recent work~\cite{rong2020dropedge} suggests that randomly removing some edges from $\tilde{G}$ retards the convergence speed of over-smoothing. Let $\tilde{\mathbf{P}}_{\mathrm{drop}}$ denote the renormalized graph convolution matrix with some edge removed at random, the vanilla GCN equipped with DropEdge is defined as
\begin{equation}
    \label{eqn:dropgcn}
    \mathbf{H}^{(\ell+1)} = \sigma\left( \tilde{\mathbf{P}}_{\mathrm{drop}}  \mathbf{H}^{(\ell)}\mathbf{W}^{(\ell)}\right).
\end{equation}

\section{GCNII  Model}
\label{sec:model}
It is known~\cite{pmlr-v97-wu19e}  that by stacking $K$ layers, the
vanilla GCN simulates a polynomial filter $\left(\sum_{\ell=0}^{K}
  \theta_{\ell} \tilde{\mathbf{L}}^{\ell}\right) \mathbf{x}$ of order $K$ with fixed
coefficients $\theta$ on the graph spectral
domain of $\tilde{G}$. The fixed coefficients limit the expressive
power of a multi-layer GCN model and thus leads to over-smoothing.
To extend GCN to a truly deep model, we need to
enable GCN to  express a $K$ order polynomial
filter with {\em arbitrary} coefficients. We show this can be achieved
by two simple
techniques: {\em Initial residual connection} and  {\em Identity
  mapping}.  Formally, we define the $\ell$-th layer of GCNII  as 
    \begin{equation}
    \label{eqn:gcnii_analysis}
    \hspace{-0.7mm} \mathbf{H}^{(\ell+1)} \hspace{-0.7mm}= \hspace{-0.7mm}
   \sigma  \hspace{-0.7mm}\left(  \hspace{-0.7mm}\left( \hspace{-0.5mm}  (1  \hspace{-0.7mm}-  \hspace{-0.7mm}\alpha_\ell)\tilde{\mathbf{P}}
        \mathbf{H}^{(\ell)}  \hspace{-0.7mm} +  \hspace{-0.7mm}
        \alpha_\ell\mathbf{H}^{(0)}  \hspace{-0.7mm}\right)  \hspace{-0.7mm}
      \left(  \hspace{-0.5mm}   (1  \hspace{-0.7mm} -  \hspace{-0.7mm}\beta_\ell) \mathbf{I}_n \hspace{-0.7mm} +
        \hspace{-0.7mm} \beta_\ell \mathbf{W}^{(\ell)}  \hspace{-0.7mm}\right)  \hspace{-0.7mm}\right),
  \end{equation}
where $\alpha_\ell$ and $\beta_\ell$ are two hyperparameters to be
discussed later.
Recall that  \( \tilde{\mathbf{P}} = \tilde{\mathbf{D}}^{-1 / 2}
\tilde{\mathbf{A}} \tilde{\mathbf{D}}^{-1 / 2}\) is the graph
convolution matrix with the renormalization trick.
Note that compared to the vanilla GCN model (equation~\eqref{eqn:gcn}), we make two
modifications: 1) We combine the smoothed representation $\tilde{\mathbf{P}}
        \mathbf{H}^{(\ell)}$ with an initial residual connection to the
first layer $\mathbf{H}^{(0)}$; 2) We add an identity
  mapping  $\mathbf{I}_n $ to the $\ell$-th weight matrix
$\mathbf{W}^{(\ell)} $. 

\paragraph{Initial residual connection.}
To simulate the skip connection in
ResNet~\cite{DBLP:conf/cvpr/HeZRS16}, 
\cite{DBLP:conf/iclr/KipfW17} proposes {\em residual connection} that combines  the
smoothed representation $\tilde{\mathbf{P}}
\mathbf{H}^{(\ell)}$ with $\mathbf{H}^{(\ell)}$. However,  it is also
shown in \cite{DBLP:conf/iclr/KipfW17} that such residual connection only partially relieves the
over-smoothing problem; the performance of the model still degrades as
we stack more layers.

We propose that, instead of using a residual
connection to carry the information from the previous layer, we
construct a connection to the initial representation
$\mathbf{H}^{(0)}$. The initial residual connection ensures that that the final
representation of each node retains at least a fraction of
$\alpha_\ell $ from the input
layer even if we stack many
layers.  In practice, we can simply set $\alpha_\ell =
0.1$ or $0.2$ so that the final representation of each node consists
of at least a fraction of the input feature. 
We also note that  $\mathbf{H}^{(0)}$ does not necessarily have to be the
feature matrix $\mathbf{X}$.  If the feature dimension $d$ is large,
we can apply a fully-connected neural network on $\mathbf{X}$ to obtain a lower-dimensional
initial representation $\mathbf{H}^{(0)}$ before the forward
propagation.

Finally, we recall that APPNP~\cite{klicpera_predict_2019} 
employs a similar approach to the initial residual connection  in the context of Personalized
PageRank. However,  \cite{klicpera_predict_2019}  also shows that performing multiple
non-linearity operations to the feature matrix will lead to
overfitting and thus results in
the performance drop.  Therefore,
APPNP applies a linear combination between different layers and thus remains a shallow
model. This suggests that the idea of initial
residual alone is not sufficient to
extend GCN to a deep model. 

\paragraph{Identity mapping.} To amend the deficiency of APPNP, we
borrow the idea of identity mapping from ResNet.
At the $\ell$-th layer, we add an identity matrix $
\mathbf{I}_n $ to the
weight matrix $\mathbf{W}^{(\ell)} $. In the following, we summarize the motivations for
introducing 
identity mapping into our model.

\begin{itemize}
\item Similar to the motivation of ResNet~~\cite{DBLP:conf/cvpr/HeZRS16},
identity mapping ensures that a deep GCNII model
achieves at
least the same performance as its shallow version does. In particular, by setting
$\beta_\ell$ sufficiently small, deep GCNII ignores the weight matrix
$\mathbf{W}^{(\ell)} $ and essentially simulates APPNP (equation~\eqref{eqn:APPNP}).

\item It has been observed that frequent interaction between
  different dimensions of the feature
  matrix~\cite{klicpera_predict_2019} degrades the performance of the
  model in semi-supervised tasks.  Mapping the
smoothed representation $\tilde{\mathbf{P}}
\mathbf{H}^{(\ell)}$  directly to the output reduces such interaction.

\item Identity
mapping is proved to be particularly useful in  semi-supervised tasks.
It is shown in~\cite{DBLP:conf/iclr/HardtM17} that a linear ResNet of the form
$\mathbf{H}^{(\ell+1)}=\mathbf{H}^{(\ell)}\left(\mathbf{W}^{(\ell)}+\mathbf{I}_n\right)$
satisfies the following properties: 
  1) The optimal weight matrices $\mathbf{W}^{(l)}$ have small norms;
  2)  The only critical point is the global minimum.  The
  first property allows us to put strong regularization on
  $\mathbf{W}^{\ell}$ to avoid over-fitting, while the later is desirable in
  semi-supervised tasks where training data is limited.
  
\item
\cite{oono2020graph} theoretically proves that the
node features of a $K$-layer GCNs will converge to a subspace and incur
information loss. In particular, the rate of convergence depends on   $s^K$, where $s$ is the maximum singular value of
  the weight matrices $\mathbf{W}^{(\ell)}$, $\ell=0, \ldots, K-1$. By replacing
  $\mathbf{W}^{(\ell)}$ with $(1-\beta_\ell) \mathbf{I}_n + \beta_\ell \mathbf{W}^{(\ell)} 
 $ and imposing regularization on $\mathbf{W}^{(\ell)} $,
  we force the norm of $\mathbf{W}^{(\ell)} $ to be small.
  Consequently,  the singular values of $(1-\beta_\ell) \mathbf{I}_n + \beta_\ell \mathbf{W}^{(\ell)} 
 $ will be close to $1$. Therefore, the maximum
  singular value $s$ will also be close to $1$, which implies that $s^K$
  is large, and the information loss is relieved.
\end{itemize}
The principle of setting $\beta_\ell$ is to ensure  the decay of the
weight matrix adaptively increases as we stack more layers. 
  In practice, we set $\beta_\ell = \log(\frac{\lambda}{\ell}+1)
  \approx \frac{\lambda}{\ell}$,  where $\lambda$ is a
  hyperparameter.

\paragraph{Connection to  iterative shrinkage-thresholding.}
Recently, there has been work on optimization-inspired network structure design \cite{DBLP:conf/cvpr/ZhangG18,papyan2017convolutional}. The idea is that a feedforward neural network can be considered as an iterative optimization algorithm to minimize some function, and it was hypothesized that better optimization algorithms might lead to better network structure \cite{li2018optimization}. Thus, theories in numerical optimization algorithms may inspire the design of better and more interpretable network structures. As we will show next, the use of identity mappings in our structure is also well-motivated from this.  We consider the LASSO objective:
\begin{equation*}
\min_{x\in \mathcal{R}^n} \frac{1}{2} \|\mathbf{Bx} - \mathbf{y}\|_2^2 + \lambda \|\mathbf{x}\|_1.
\end{equation*}
Similar to  compressive sensing, we consider $\mathbf{x}$ as the signal we are trying to recover, $\mathbf{B}$ as the measurement matrix, and $\mathbf{y}$ as the signal we observe. In our setting, $\mathbf{y}$ is the original feature of a node, and $\mathbf{x}$ is the node embedding the network tries to learn. As opposed to standard regression models, the design matrix $\mathbf{B}$ is unknown parameters and will be learned through back propagation. So, this is in the same spirit as the sparse coding problem, which has been used to design and to analyze CNNs \cite{papyan2017convolutional}. Iterative shrinkage-thresholding algorithms are effective for solving the above optimization problem, in which the update in the $(t+1)$th iteration is:
\begin{equation*}
\mathbf{x}^{t+1} = P_{\mu_t \lambda}\left( \mathbf{x}^{t} - \mu_t \mathbf{B}^T\mathbf{B} \mathbf{x}^{t} +  \mu_t \mathbf{B}^T \mathbf{y}  \right),
\end{equation*}
Here $\mu_t$ is the step size, and $P_{\beta} (\cdot)$ (with $\beta >0$) is the entry-wise soft thresholding function:
\[
P_{\theta} (z) = \left\{
\begin{array}{lr}
z-\theta, & \text{if } z\ge \theta\\
0, & \text{if } |z|<\theta \\
z+\theta, & \text{if } z\le -\theta
\end{array} \right. .
\]
Now, if we reparameterize $-\mathbf{B^TB}$ by $\mathbf{W}$, the above update formula becomes quite similar to the one used in our method. More spopposeecifically, we have $\mathbf{x}^{t+1} = P_{\mu_t \lambda}\left( (\mathbf{I}+\mu_t \mathbf{W}) \mathbf{x}^{t} +  \mu_t \mathbf{B}^T \mathbf{y}  \right)$, where the term $\mu_t \mathbf{B}^T \mathbf{y}$ corresponds to the initial residual,  and $\mathbf{I}+\mu_t \mathbf{W}$ corresponds to the identity mapping 
in our model~\eqref{eqn:gcnii_analysis}. The soft thresholding operator acts as the nonlinear activation function, which is similar to the effect of ReLU activation. In conclusion, our network structure, especially the use of identity mapping is well-motivated from iterative shrinkage-thresholding algorithms for solving LASSO.

\section{Spectral Analysis}
\label{sec:analysis}
\subsection{Spectral analysis of multi-layer GCN. }
We consider the following GCN model
with residual connection: 
\begin{equation}
    \label{eqn:gcn_res}
    \mathbf{H}^{(\ell+1)} = \sigma\left( \left( \tilde{\mathbf{P}}
      \mathbf{H}^{(\ell)} +  \mathbf{H}^{(\ell)} \right)\mathbf{W}^{(\ell)}\right).
\end{equation}
Recall that  \( \tilde{\mathbf{P}} = \tilde{\mathbf{D}}^{-1 / 2}
\tilde{\mathbf{A}} \tilde{\mathbf{D}}^{-1 / 2}\) is the graph
convolution matrix with the renormalization trick.
\cite{wang2019improving} points out that equation~\eqref{eqn:gcn_res}
simulates a {\em lazy random walk} with the transition matrix  ${\mathbf{I}_n+\tilde{\mathbf{D}}^{-1 / 2}
      \tilde{\mathbf{A}} \tilde{\mathbf{D}}^{-1 / 2} \over 2}$.  Such a lazy random walk eventually
converges to the stationary state and thus leads to
over-smoothing. We now derive the closed-form of the
stationary vector and analyze the rate of such convergence. Our analysis suggests that the converge
rate of an individual node depends on its degree, and we conduct
experiments to back up this theoretical finding. In particular, we have the following Theorem.

\begin{theorem}
  \label{thm:GCN}
  Assume the self-looped graph $\tilde{G}$ is connected. 
  Let $\mathbf{h}^{(K)}= \left({\mathbf{I}_n+\tilde{\mathbf{D}}^{-1 / 2}
      \tilde{\mathbf{A}} \tilde{\mathbf{D}}^{-1 / 2} \over
      2}\right)^K\cdot \mathbf{x}$ denote the representation by applying a
$K$-layer renormalized graph convolution with residual connection to
a graph signal $\mathbf{x}$. Let $\lambda_{\tilde{G}}$ denote the
spectral gap of the self-looped graph $\tilde{G}$, that is,  the least nonzero eigenvalue of the
normalized Laplacian  \(\tilde{\mathbf{L}}=\mathbf{I}_{n}-
\tilde{\mathbf{D}}^{-1 / 2} \tilde{\mathbf{A}} \tilde{\mathbf{D}}^{-1
  / 2}\).  We have
  
1) As $K$ goes to infinity, $\mathbf{h}^{(K)}$ converges to $\bm{\pi} =
{\left<\tilde{\mathbf{D}}^{1 / 2} \mathbf{1}, \mathbf{x} \right> \over
  2m+n} \cdot
\tilde{\mathbf{D}}^{1 / 2} \mathbf{1}
$, where $\mathbf{1}$ denotes an all-one vector.

2) The convergence rate is determined by
\begin{equation}
  \label{eqn:GCN_THM}
  \mathbf{h}^{(K)} = \bm{\pi} \pm\left( \sum_{i=1}^n x_i
  \right) \cdot \left( 1-
    {\lambda_{\tilde{G}}^2 \over 2}\right)^K \cdot  \mathbf{1}.
  \end{equation}

\end{theorem}

Recall that $m$ and $n$ are the number of nodes and edges in the
original graph $G$.
We use  the operator $\pm$ to indicate that  for each entry $ \mathbf{h}^{(K)} (j)$ and $\bm{\pi}(j)
$, $j=1, \ldots, n$, 
$$\left| \mathbf{h}^{(K)} (j) -  \bm{\pi}(j) \right| \le \left( \sum_{i=1}^n x_i
  \right) \cdot \left( 1-
    {\lambda_{\tilde{G}}^2 \over 2}\right)^K .$$
  The proof of Theorem~\ref{thm:GCN} can be found in the supplementary
materials. 
There are two consequences from Theorem~\ref{thm:GCN}. First of all,
it suggests that the $K$-th representation  of GCN $\mathbf{h}^{(K)}$
converges to a vector  $\bm{\pi} =
{\left<\tilde{\mathbf{D}}^{1 / 2} \mathbf{1}, \mathbf{x} \right> \over
  2m+n} \cdot
\tilde{\mathbf{D}}^{1 / 2} \mathbf{1}
$. Such convergence leads to over-smoothing as the vector
$\bm{\pi}$  only carries the two kinds of information: the degree of
each node, and  the inner product between the
initial 
signal $\mathbf{x}$ and vector $\mathbf{D}^{1 / 2}
\mathbf{1}$. 

\paragraph{Convergence rate and node degree.}
Equation~\eqref{eqn:GCN_THM} suggests that 
the  converge rate depends on the summation of feature entries
$\sum_{i=1}^n x_i$ and the spectral gap
$\lambda_{\tilde{G}}$. If we take a closer look at the {\em relative converge
rate} for an individual node $j$, we can express its final representation $\mathbf{h}^{(K)}(j) $
as 
\begin{align*}
  \hspace{-2mm}\mathbf{h}^{(K)}(j) \hspace{-0.7mm}=
  \hspace{-0.7mm}\sqrt{d_j+1}   \hspace{-0.7mm} \left(
  \hspace{-0.7mm}\sum_{i=1}^n{\sqrt{d_i  \hspace{-0.7mm} +
  \hspace{-0.7mm} 1} \over 2m  \hspace{-0.7mm} +  \hspace{-0.7mm}  n}x_i
       \hspace{-0.7mm}  \pm  \hspace{-0.7mm}{ \sum_{i=1}^n x_i
  \hspace{-0.7mm} \left( 1  \hspace{-0.7mm}-  \hspace{-0.7mm} {\lambda_{\tilde{G}}^2 \over 2}\right)^K\over \sqrt{d_j+1}}  \hspace{-0.7mm}\right).
 \end{align*}

 This suggests that if a node $j$ has a higher degree of $d_j$ (and hence a
 larger $ { \sqrt{d_j+1}}$), its representation $\mathbf{h}^{(K)}(j) $ converges
 faster to the stationary state $\bm{\pi}(j)$. Based on this fact, we make the
 following conjecture.
 \begin{conjecture}
   \label{con:degree}
   Nodes with higher degrees are more likely to suffer from over-smoothing.
 \end{conjecture}
We will verify Conjecture~\ref{con:degree} on real-world datasets
 in our experiments.
 
\subsection{Spectral analysis of GCNII}
We consider the spectral domain of the self-looped graph $\tilde{G}$.
Recall that a polynomial filter of order $K$ on a graph signal
$\mathbf{x}$ is defined as $\left(\sum_{\ell=0}^{K}
  \theta_{\ell} \tilde{\mathbf{L}}^{\ell}\right) \mathbf{x}$, where
$\tilde{\mathbf{L}}$ is the normalized Laplacian matrix of
$\tilde{G}$ and $\theta_k$'s are the polynomial coefficients. \cite{pmlr-v97-wu19e} proves that a $K$-layer GCN
simulates a polynomial filter of order $K$ with fixed coefficients
$\theta$. As we shall prove later, such fixed coefficients limit the expressive
power of GCN and thus leads to over-smoothing. On the other hand, 
we show a $K$-layer GCNII model can  express a $K$ order polynomial
 filter with arbitrary coefficients. 
\begin{theorem}
  \label{thm:GCNII}
Consider the self-looped graph $\tilde{G}$ and a graph signal $\mathbf{x}$.
  A $K$-layer GCNII can  express a $K$ order polynomial filter $\left(\sum_{\ell=0}^{K}
    \theta_{\ell} \tilde{\mathbf{L}}^{\ell}\right) \mathbf{x}$ with arbitrary
  coefficients $\theta$.
  \end{theorem}
The proof of Theorem~\ref{thm:GCNII} can be found in the supplementary
materials. Intuitively, the parameter $\beta$ allows GCNII to
simulate the coefficient $\theta_\ell$ of the polynomial filter. 

\paragraph{Expressive power and over-smoothing.}
The ability to express a polynomial filter with
arbitrary coefficients is essential for preventing over-smoothing. To see why this is the case, recall that Theorem~\ref{thm:GCN} suggests a $K$-layer vanilla GCN simulates a fixed $K$-order polynomial filter $\tilde{\mathbf{P}}^{K}\mathbf{x}$, where $\tilde{\mathbf{P}}$ is the renormalized graph convolution matrix. Over-smoothing is caused by the fact that $\tilde{\mathbf{P}}^{K}\mathbf{x}$ converges to a distribution isolated from the input feature $\mathbf{x}$ and thus incuring gradient vanishment. DropEdge~\cite{rong2020dropedge} slows down the rate of convergence, but eventually will fail as $K$ goes to infinity. 

On the other hand, Theorem~\ref{thm:GCNII} suggests that deep GCNII converges to a distribution that carries information from both the input feature and the graph structure. This property alone ensures that GCNII will not suffer from over-smoothing even if the number of layers goes to infinity. More precisely, 
Theorem~\ref{thm:GCNII} states that a $K$-layer GCNII can express $\mathbf{h}^{(K)}  = \left(\sum_{\ell=0}^{K}
\theta_{\ell} \tilde{\mathbf{L}}^{\ell}\right) \cdot\mathbf{x}$ with arbitrary coefficients $\theta$. Since the renormalized graph convolution matrix $\tilde{\mathbf{P}} = \mathbf{I}_n - \tilde{\mathbf{L}}$, it follows that $K$-layer GCNII can express $\mathbf{h}^{(K)}  = \left(\sum_{\ell=0}^{K}
\theta'_{\ell} \tilde{\mathbf{P}}^{\ell}\right)\cdot \mathbf{x}$ with arbitrary coefficients $\theta'$.  Note that with a proper choice of $\theta'$, $\mathbf{h}^{(K)}$ can carry information from both the input feature and the graph structure even with K going to infinity. For example, APPNP~\cite{klicpera_predict_2019} and GDC~\cite{DBLP:conf/nips/KlicperaWG19} set $\theta'_i = \alpha(1-\alpha)^i$ for some constant $0<\alpha <1$. As K goes to infinity, $\mathbf{h}^{(K)}  = \left(\sum_{\ell=0}^{K}
\theta'_{\ell} \tilde{\mathbf{P}}^{\ell}\right)\cdot \mathbf{x}$ converges to the Personalized PageRank vector of $\mathbf{x}$, which is a function of both the adjacency matrix $\tilde{\mathbf{A}}$ and the input feature vector $\mathbf{x}$. The difference between GCNII and APPNP/GDC is that 1) the coefficient vector theta in our model is learned from the input feature and the label, and 2) we impose a ReLU operation at each layer. 

\section{Other Related Work}
\label{sec:related}

Spectral-based GCN has been extensively studied for the past few years. \cite{DBLP:conf/aaai/LiWZH18} improves flexibility by learning a task-driven adaptive graph for each graph data while training. \cite{xu2018graph} uses the graph wavelet basis instead of the Fourier basis to improve sparseness and locality. Another line of works focuses on the attention-based GCN model~\cite{velickovic2018graph,180303735,DBLP:conf/uai/ZhangSXMKY18}, which learn the edge weights at each layer based on node features. \cite{DBLP:conf/icml/Abu-El-HaijaPKA19} learn neighborhood mixing relationships by mixing of neighborhood information at various distances but still uses a two-layer model. \cite{DBLP:conf/icml/GaoJ19,DBLP:conf/icml/LeeLK19} devote to extend pooling operations to graph neural network. For unsupervised information, \cite{DBLP:conf/iclr/VelickovicFHLBH19} train graph convolutional encoder through maximizing mutual information. \cite{Pei2020GeomGCN} build structural neighborhoods in the latent space of graph embedding for aggregation to extract more structural information. \cite{DBLP:journals/dase/DaveZCH19} uses a single representation vector to capture both topological information and nodal attributes in graph embedding. Many of the sampling-based methods proposed to improve the scalability of GCN. \cite{hamilton2017inductive} uses a fixed size of neighborhood samples through layers, \cite{DBLP:conf/iclr/ChenMX18,DBLP:conf/nips/Huang0RH18} propose efficient variants based on importance sampling. \cite{DBLP:conf/kdd/ChiangLSLBH19} construct minibatch based on graph clustering.

\section{Experiments}
\label{sec:exp}
\begin{table}[t]
    \caption{Dataset statistics.}
    \label{dataset-table}
    \vskip 0.10in

    \begin{tabular}{lrrrr}
    \toprule
    Dataset & Classes & Nodes & Edges & Features \\
    \midrule
    Cora        & 7   &  2,708 &   5,429 & 1,433 \\
    Citeseer    & 6   &  3,327 &   4,732 & 3,703 \\
    Pubmed      & 3   & 19,717 &  44,338 & 500 \\
    Chameleon   & 4   &  2,277 &  36,101 & 2,325 \\
    Cornell     & 5   &    183 &    295  & 1,703 \\
    Texas       & 5   &    183 &    309  & 1,703 \\
    Wisconsin   & 5   &    251 &    499  & 1,703 \\
    PPI         & 121 & 56,944 & 818,716 & 50 \\
    \bottomrule
    \end{tabular}

\end{table}

In this section, we evaluate the performance of GCNII against
the state-of-the-art graph neural network models on a wide variety of open graph datasets.

\paragraph{Dataset and experimental setup.}
We use three standard citation network datasets Cora, Citeseer, and
Pubmed~\cite{DBLP:journals/aim/SenNBGGE08} for semi-supervised node
classification. In these citation datasets, nodes correspond to
documents, and edges correspond to citations; each node feature corresponds to
the bag-of-words representation of the document and  belongs to one
of the academic topics. 
For full-supervised node classification, we also
include  Chameleon~\cite{1909-13021}, Cornell, Texas, and
Wisconsin~\cite{Pei2020GeomGCN}. These datasets are web networks,
where nodes and edges represent web pages and hyperlinks, respectively. The feature of each node is the bag-of-words
representation of the corresponding page. For inductive learning, we
use Protein-Protein Interaction (PPI)
networks~\cite{hamilton2017inductive},  which contains 24
graphs. Following the setting of previous work~\cite{velickovic2018graph},  we
use  20 
graphs for training, 2 graphs for validation, and the rest for
testing. Statistics of the datasets are summarized in  Table~\ref{dataset-table}.

Besides GCNII~\eqref{eqn:gcnii_analysis}, we also include GCNII*, a variant of GCNII that employs different weight matrices
for the smoothed representation
$\tilde{\mathbf{P}}\mathbf{H}^{(\ell)}$  and the  initial residual
$\mathbf{H}^{(0)}$. Formally, the $(\ell+1)$-th layer of GCNII* is
defined as
    \begin{align*}
   \mathbf{H}^{(\ell+1)} &= 
   \sigma   \left((1-\alpha_\ell)\tilde{\mathbf{P}}
        \mathbf{H}^{(\ell)}  \left((1-\beta_\ell) \mathbf{I}_n +
                                        \beta_\ell
                                         \mathbf{W}_1^{(\ell)}  \right) +\right.\\
      & \quad \left.+ 
        \alpha_\ell\mathbf{H}^{(0)}  
      \left((1-\beta_\ell) \mathbf{I}_n +
      \beta_\ell \mathbf{W}_2^{(\ell)}\right)  \right).
  \end{align*}
  As mentioned in Section~\ref{sec:model}, we set $\beta_\ell =
  \log(\frac{\lambda}{\ell}+1) \approx \lambda/\ell$, where $\lambda$ is a hyperparameter. 

  \begin{table}[t]
    % \vspace{-2mm}
    \caption{Summary of classification accuracy ($\%$) results on Cora, Citeseer, and Pubmed. 
            The number in parentheses corresponds to the number of
            layers of the model.}
    \label{semi-table}
    \vskip 0.10in
    \begin{small}
    \setlength{\tabcolsep}{1mm}{
    \begin{tabular}{llll}
    \toprule
    Method & Cora & Citeseer & Pubmed \\
    \midrule
    GCN & 81.5 & 71.1 & 79.0 \\
      GAT & 83.1 & 70.8 & 78.5 \\
      APPNP & 83.3 & 71.8 & 80.1 \\
    JKNet & 81.1 (4) & 69.8 (16)& 78.1 (32) \\
    JKNet(Drop) & 83.3 (4) & 72.6 (16)& 79.2 (32) \\
    Incep(Drop) & 83.5 (64) & 72.7 (4) & 79.5 (4) \\  

    \midrule
    GCNII & \textbf{85.5 $\pm$ 0.5} (64)& \textbf{73.4 $\pm$ 0.6} (32)& 80.2 $\pm$ 0.4 (16)\\
    GCNII* & 85.3 $\pm$ 0.2 (64)& 73.2 $\pm$ 0.8 (32)& \textbf{80.3 $\pm$ 0.4} (16)\\
    \bottomrule
    \end{tabular}}
    \end{small}
    % \end{center}
    \vskip -0.1in
  \end{table}

%  \vspace{-1mm}
 \subsection{Semi-supervised Node Classification}
%   \vspace{-1mm}
\paragraph{Setting and baselines.}
For the semi-supervised node classification task, we apply the standard fixed training/validation/testing
split~\cite{DBLP:conf/icml/YangCS16} on three datasets Cora, Citeseer, and
 Pubmed,  with 20 nodes per class for
training, 500 nodes for validation and
1,000 nodes  for testing. For baselines, we include two recent deep GNN models: 
JKNet~\cite{DBLP:conf/icml/XuLTSKJ18} and
DropEdge~\cite{rong2020dropedge}. As suggested
in~\cite{rong2020dropedge}, we equip DropEdge on three 
backbones: GCN~\cite{DBLP:conf/iclr/KipfW17}, JKNet~\cite{DBLP:conf/icml/XuLTSKJ18} and IncepGCN~\cite{rong2020dropedge}.
We also include three state-of-the-art shallow models:
GCN~\cite{DBLP:conf/iclr/KipfW17}, GAT~\cite{velickovic2018graph} and
APPNP~\cite{klicpera_predict_2019}.

We use  the Adam SGD optimizer~\cite{DBLP:journals/corr/KingmaB14}
with a learning rate of 0.01 and early stopping with a patience of 100
epochs to train GCNII and GCNII*. We set $\alpha_\ell=0.1$ and $L_{2}$ regularization to $0.0005$ for
the dense layer on all datasets. 
We perform a grid search to tune the
other hyper-parameters for models with different depths based on the
accuracy on the validation set. More details of hyper-parameters are
listed in the supplementary
materials.

\paragraph{Comparison with SOTA.}
Table~\ref{semi-table} reports the mean classification accuracy with the standard deviation on
the test nodes of GCN and GCNII after 100 runs. We reuse the metrics
already reported in~\cite{Fey/Lenssen/2019}  for GCN, GAT, and APPNP,
and the best metrics reported in~\cite{rong2020dropedge} for JKNet,
JKNet(Drop) and Incep(Drop). Our results successfully demonstrate that
GCNII and GCNII* achieves new state-of-the-art performance across all
three datasets. Notably, GCNII outperforms the previous
state-of-the-art methods by at least $2\%$. It is also worthwhile to
note that  the two recent deep models, JKNet and IncepGCN with DropEdge, do
not seem to offer significant advantages over the shallow model
APPNP. On the other hand, our method achieves this result with a 64-layer model, which
demonstrates the benefit of deep network structures.

\begin{table}[t]
%   \vspace{-2mm}
    \caption{Summary of classification accuracy ($\%$)  results with various
      depths.}
    \label{depth-table}
    \vskip 0.10in
    % \begin{center}
    \begin{small}
    \setlength{\tabcolsep}{1.mm}{
    \begin{tabular}{ll|cccccc}
        \toprule
            \multirow{2}{*}{Dataset} & \multirow{2}{*}{Method} & \multicolumn{6}{c}{Layers} \\
                                     &  & 2  & 4  & 8 & 16 & 32 & 64 \\
            \midrule
            \multirow{7}{*}{Cora}       & GCN & \textbf{81.1}  & 80.4  & 69.5  & 64.9  & 60.3  & 28.7  \\
                                        & GCN(Drop) & \textbf{82.8}  & 82.0  & 75.8  & 75.7  & 62.5  & 49.5  \\
                                        & JKNet & - & 80.2  & 80.7  & 80.2  & \textbf{81.1} & 71.5  \\
                                        & JKNet(Drop) & - & \textbf{83.3}  & 82.6  & 83.0  & 82.5 & 83.2  \\
                                        & Incep & - & 77.6  & 76.5  & 81.7  & \textbf{81.7} & 80.0  \\
                                        & Incep(Drop) & - & 82.9  & 82.5  & 83.1  & 83.1 & \textbf{83.5}  \\
                                        & GCNII & 82.2  & 82.6  & 84.2  & 84.6  & 85.4  & \textbf{85.5}  \\
                                        & GCNII* & 80.2  & 82.3  & 82.8  & 83.5  & 84.9  & \textbf{85.3}  \\
            \midrule
            \multirow{7}{*}{Citeseer}   & GCN & \textbf{70.8}  & 67.6  & 30.2  & 18.3  & 25.0  & 20.0  \\
                                        & GCN(Drop) & \textbf{72.3}  & 70.6  & 61.4  & 57.2  & 41.6  & 34.4  \\
                                        & JKNet & - & 68.7  & 67.7  & \textbf{69.8}  & 68.2 & 63.4  \\
                                        & JKNet(Drop) & - & 72.6  & 71.8  & \textbf{72.6} & 70.8 & 72.2  \\
                                        & Incep & - & 69.3  & 68.4  & \textbf{70.2}  & 68.0 & 67.5  \\
                                        & Incep(Drop) & - & \textbf{72.7}  & 71.4  & 72.5  & 72.6 & 71.0  \\
                                        & GCNII & 68.2  & 68.9  & 70.6  & 72.9  & \textbf{73.4}  & 73.4  \\
                                        & GCNII* & 66.1  & 67.9  & 70.6  & 72.0  & \textbf{73.2}  & 73.1  \\
            \midrule
            \multirow{7}{*}{Pubmed}     & GCN & \textbf{79.0}  & 76.5  & 61.2  & 40.9  & 22.4  & 35.3  \\
                                        & GCN(Drop) & \textbf{79.6}  & 79.4  & 78.1  & 78.5  & 77.0  & 61.5  \\
                                        & JKNet & - & 78.0  & \textbf{78.1}  & 72.6  & 72.4 & 74.5  \\
                                        & JKNet(Drop) & - & 78.7  & 78.7  & 79.1  & \textbf{79.2} & 78.9 \\
                                        & Incep & - & 77.7  & \textbf{77.9}  & 74.9  & OOM & OOM  \\
                                        & Incep(Drop) & - & \textbf{79.5}  & 78.6  & 79.0  & OOM & OOM  \\
                                        & GCNII & 78.2  & 78.8  & 79.3  & \textbf{80.2}  & 79.8  & 79.7  \\
                                        & GCNII* & 77.7  & 78.2  & 78.8  & \textbf{80.3}  & 79.8  & 80.1  \\
        \bottomrule
    \end{tabular}}
    \end{small}
    % \end{center}
    \vskip -0.1in
\end{table}

\begin{table}[t]
 % \vspace{-1mm}
    \caption{Summary of Micro-averaged F1 scores on PPI.}
    \label{ppi-table}
    \vskip 0.15in
    \begin{center}
    % \begin{small}
    % \setlength{\tabcolsep}{1.mm}{
    \begin{tabular}{ll}
    \toprule
    Method & PPI\\
    \midrule
    GraphSAGE~\cite{hamilton2017inductive} & 61.2 \\
    VR-GCN~\cite{DBLP:conf/icml/ChenZS18} & 97.8 \\
    GaAN~\cite{DBLP:conf/uai/ZhangSXMKY18} & 98.71 \\
    GAT~\cite{velickovic2018graph} & 97.3 \\
    JKNet~\cite{DBLP:conf/icml/XuLTSKJ18} & 97.6 \\
    GeniePath~\cite{DBLP:conf/aaai/LiuCLZLSQ19} & 98.5 \\
    Cluster-GCN~\cite{DBLP:conf/kdd/ChiangLSLBH19} & 99.36 \\
    \midrule
    GCNII & 99.53 $\pm$ 0.01 \\
    GCNII* & \textbf{99.56} $\pm$ \textbf{0.02}\\
    \bottomrule
    \end{tabular}
    % \end{small}
    \end{center}
    \vskip -0.1in
  \end{table}
  
\begin{table*}[t]
    \caption{Mean classification accuracy of full-supervised node classification. 
            }
    \label{fulltrain-table}
    \vskip 0.10in
    \begin{center}
    % \begin{small}
    \begin{tabular}{lllllllll}
    \toprule
    Method & Cora & Cite. & Pumb. & Cham. & Corn. & Texa. & Wisc. \\
    \midrule
    GCN & 85.77 & 73.68 & 88.13 & 28.18 & 52.70 & 52.16 & 45.88 \\
    GAT & 86.37 & 74.32 & 87.62 & 42.93 & 54.32 & 58.38 & 49.41 \\
    Geom-GCN-I & 85.19 & \textbf{77.99} & 90.05 & 60.31 & 56.76 & 57.58 &
                                                                     58.24 \\
    Geom-GCN-P & 84.93 & 75.14 & 88.09 & 60.90 & 60.81 & 67.57 &
                                                                       64.12 \\
    Geom-GCN-S & 85.27 & 74.71 & 84.75 & 59.96 & 55.68 & 59.73 & 56.67 \\
    APPNP & 87.87 & 76.53 & 89.40 & 54.3 & 73.51 & 65.41 & 69.02 \\
    JKNet & 85.25 (16)& 75.85 (8)& 88.94 (64) & 60.07 (32) & 57.30 (4) & 56.49 (32) & 48.82 (8) \\
    JKNet(Drop) & 87.46 (16)& 75.96 (8)& 89.45 (64) & 62.08 (32) & 61.08 (4) & 57.30 (32) & 50.59 (8) \\
    Incep(Drop) & 86.86 (8)& 76.83 (8)& 89.18 (4) & 61.71 (8) & 61.62 (16) & 57.84 (8) & 50.20 (8) \\
    \midrule
    GCNII & \textbf{88.49} (64) & 77.08 (64) & 89.57 (64) & 60.61 (8) & 74.86 (16) & 69.46 (32) & 74.12 (16) \\
    GCNII* & 88.01 (64) & 77.13 (64) & \textbf{90.30} (64) & \textbf{62.48} (8) & \textbf{76.49} (16) & \textbf{77.84} (32) & \textbf{81.57} (16)\\
    \bottomrule
    \end{tabular}
    % \end{small}
    \end{center}
    \vskip -0.1in
  \end{table*}

\paragraph{A detailed comparison with other deep models.}
Table~\ref{depth-table} summaries the results for the deep models with
various numbers of layers. We reuse the best-reported results for JKNet,
JKNet(Drop) and
Incep(Drop)~$\footnote{https://github.com/DropEdge/DropEdge}$.
We observe that on Cora and Citeseer, the performance of  GCNII and GCNII* consistently
improves as we increase the number of layers. On Pubmed, GCNII and GCNII* achieve
the best results with 16 layers, and maintain similar performance as
we increase the network depth to 64. We attribute this quality to the
identity mapping technique. Overall, the results suggest
that with initial residual and identity mapping, we can resolve the 
over-smoothing problem and extend the vanilla GCN into a truly deep model. 
On the other hand, the
performance of GCN
with DropEdge and JKNet drops rapidly as the number of layers exceeds 32,
which means they still suffer from over-smoothing. 

\subsection{Full-Supervised Node Classification}
We now evaluate GCNII in the task of full-supervised node
classification.  Following the setting in~\cite{Pei2020GeomGCN}, we
use 7 datasets: Cora, Citeseer, Pubmed, Chameleon, Cornell, Texas, and
Wisconsin.
For each datasets, we randomly split
nodes of each class into 60\%, 20\%, and 20\% for training, validation
and testing, and measure the performance of
all models on the test sets over 10 random splits, as suggested in~\cite{Pei2020GeomGCN}.
We
fix the learning rate to 0.01, dropout rate to 0.5 and the number of hidden units to 64 on all
datasets and perform a hyper-parameter search to tune other
hyper-parameters based on the validation set. Detailed configuration of all model
for full-supervised node classification can be found in the
supplementary materials.
% (See Section~\ref{appendix-hyperparameters}).
Besides the
previously mentioned baselines, we also include three variants
of Geom-GCN~\cite{Pei2020GeomGCN} as they are the state-of-the-art
models on these datasets.

Table~\ref{fulltrain-table} reports the mean classification accuracy
of each model. We reuse the
metrics already reported in~\cite{Pei2020GeomGCN} for GCN, GAT, and
Geom-GCN.
We observe that GCNII and GCNII* achieves new state-of-the-art results
on 6 out of 7 datasets, which demonstrates the superiority of the deep
GCNII framework. Notably, GCNII* outperforms APPNP by over 12\% on the
Wisconsin dataset. This result suggests that by introducing
non-linearity into each layer, the predictive power of GCNII is
stronger than that of the linear model APPNP.

\subsection{Inductive Learning}
For the inductive learning task, we apply  9-layer GCNII and GCNII* models with
2048 hidden units on the PPI dataset. We fix the following sets of
hyperparameters: $\alpha_\ell=0.5$, $\lambda=1.0$ and learning rate of
0.001. Due to the large volume of training data, we set the dropout rate to 0.2
and the weight decay to
zero. Following~\cite{velickovic2018graph},  we also add a skip
connection from the $\ell$-th layer to the $(\ell+1)$-th layer of
GCNII and GCNII* to speed up the convergence of the training process. 
We compare GCNII with the
following state-of-the-art methods:
GraphSAGE~\cite{hamilton2017inductive},
VR-GCN~\cite{DBLP:conf/icml/ChenZS18},
GaAN~\cite{DBLP:conf/uai/ZhangSXMKY18},
GAT~\cite{velickovic2018graph}, JKNet~\cite{DBLP:conf/icml/XuLTSKJ18},
GeniePath~\cite{DBLP:conf/aaai/LiuCLZLSQ19},
Cluster-GCN~\cite{DBLP:conf/kdd/ChiangLSLBH19}. The metrics are summarized
in Table~\ref{ppi-table}.

In concordance with our expectations, the results show that GCNII and GCNII*
achieve new state-of-the-art performance on PPI. In particular, 
GCNII achieves this performance with a 9-layer model, while the number of
layers with all baseline
models are less or equal to 5. This suggests that larger predictive
power can also be leveraged by increasing the network depth in the
task of inductive learning.
 
\begin{figure*}[!t]
  %  \vskip 0.1in
    \begin{center}
    % \begin{small}
     \includegraphics[height=42mm]{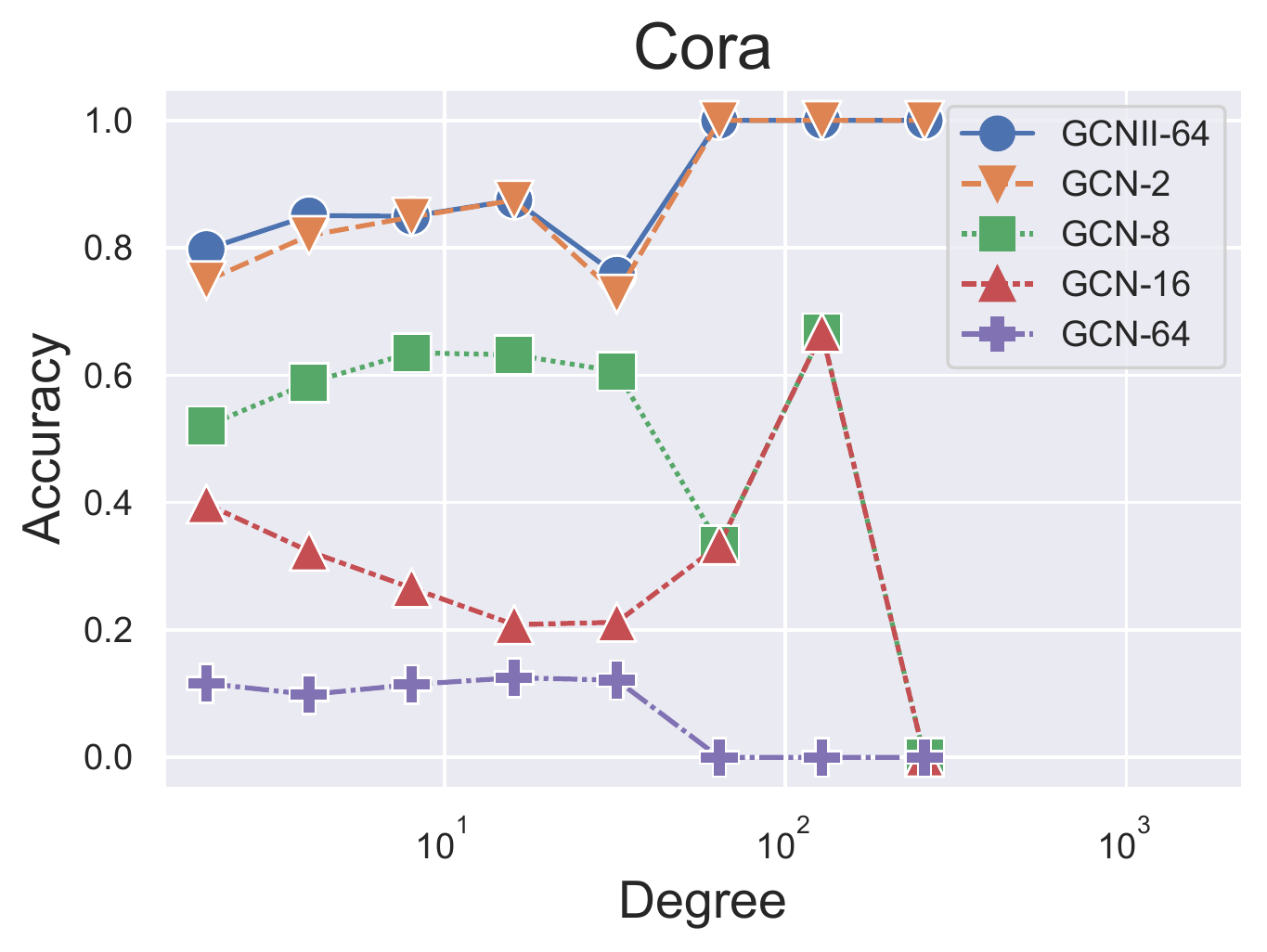}
     \includegraphics[height=42mm]{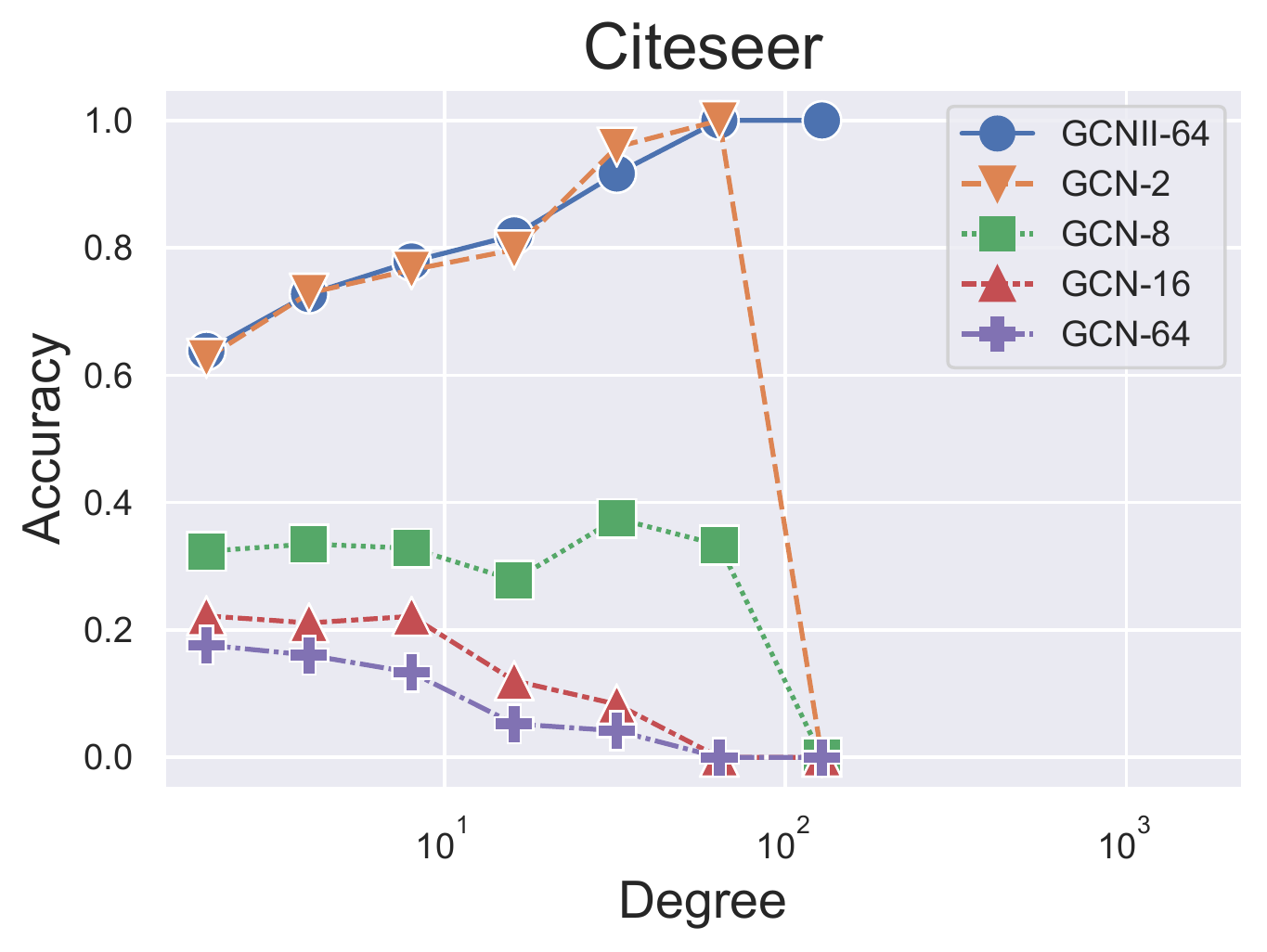}
     \includegraphics[height=42mm]{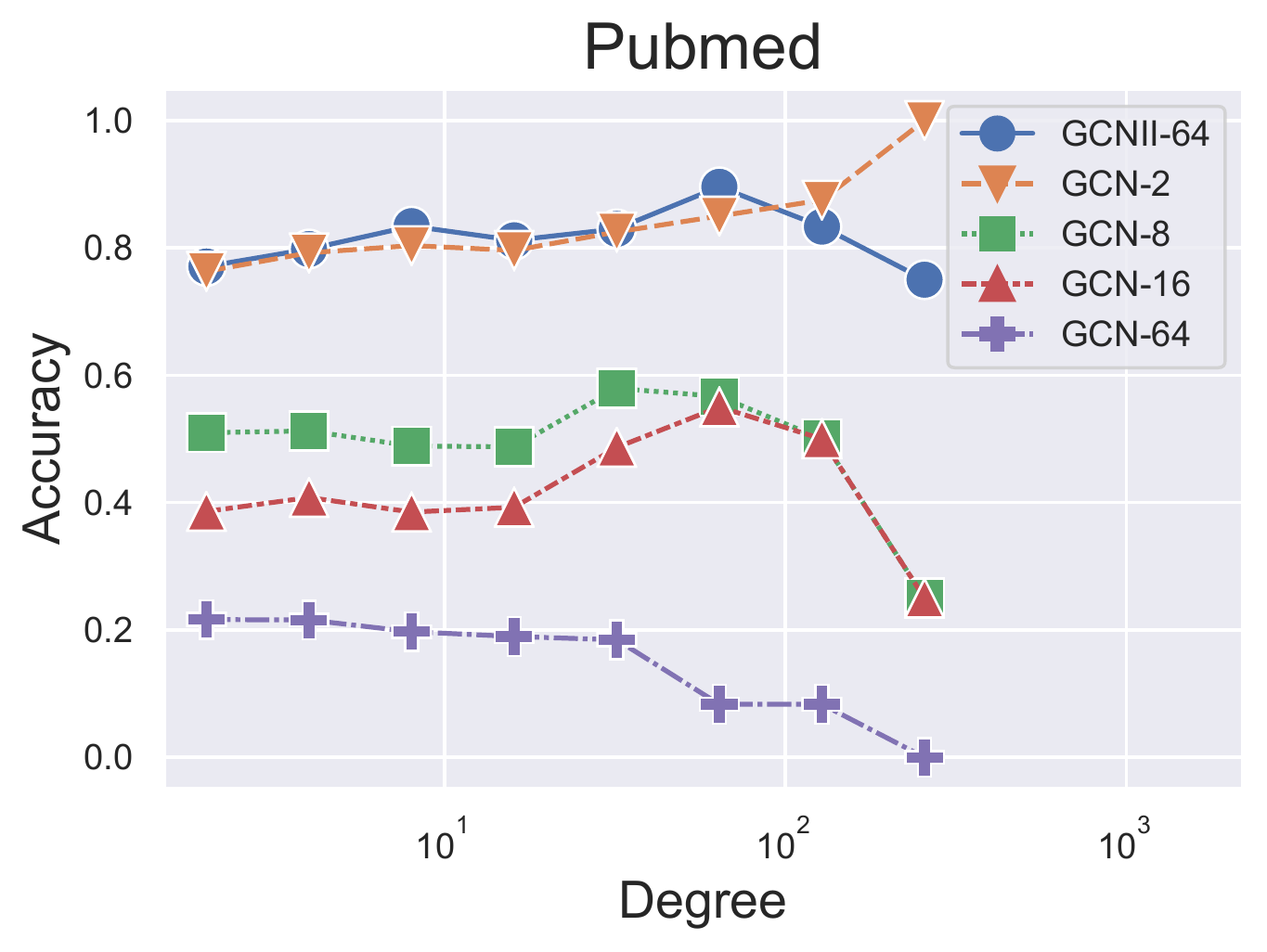}
         \vspace{-2mm}
     \caption{Semi-supervised node classification accuracy v.s. degree.} 
     \label{fig:degree_acc}
    % \end{small}
    \end{center}
    \vskip -0.1in
  \end{figure*}

  \begin{figure*}[!t]
   % \vskip 0.1in
    \begin{center}
    % \begin{small}
        \includegraphics[height=42mm]{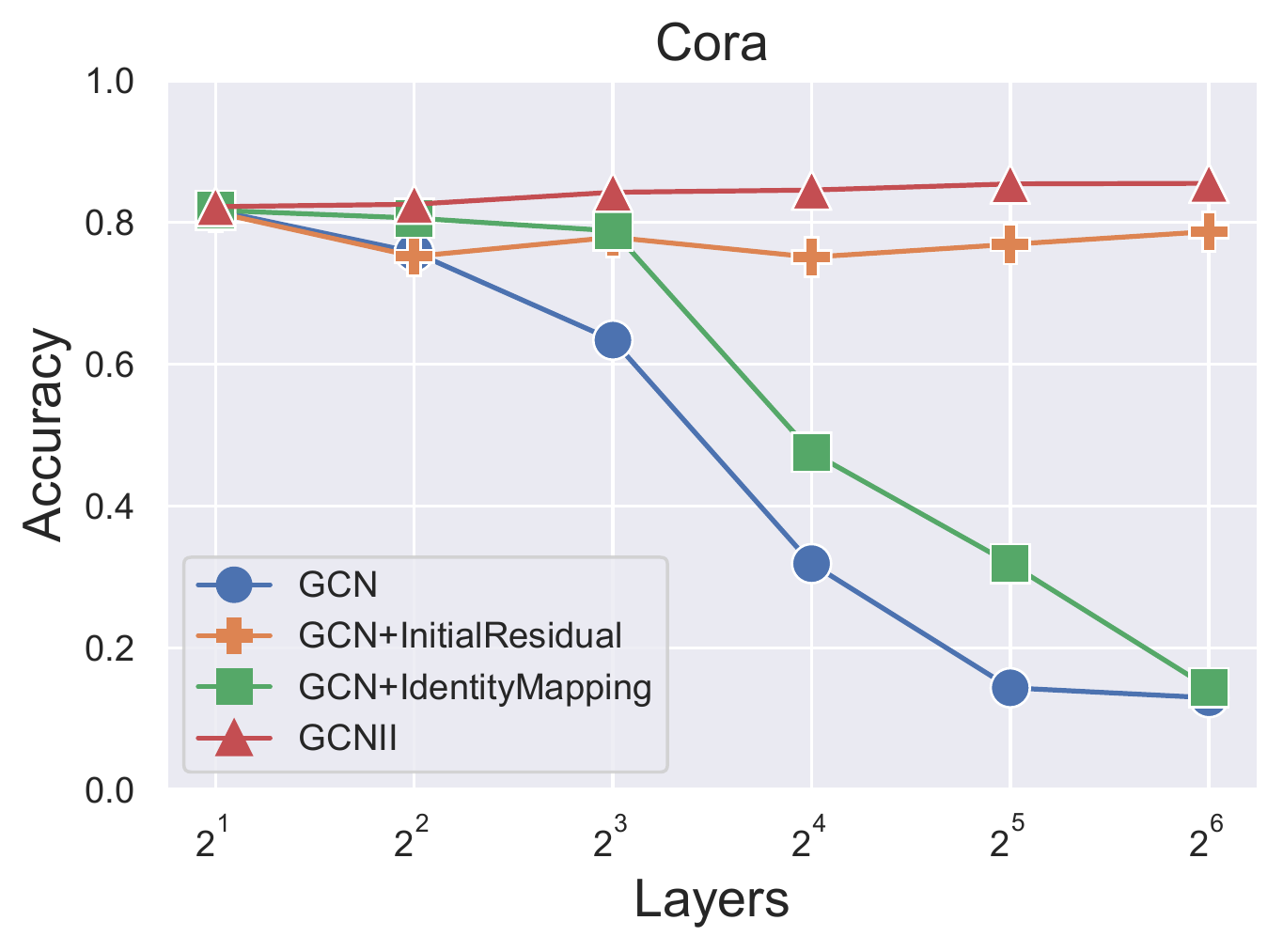}
        \includegraphics[height=42mm]{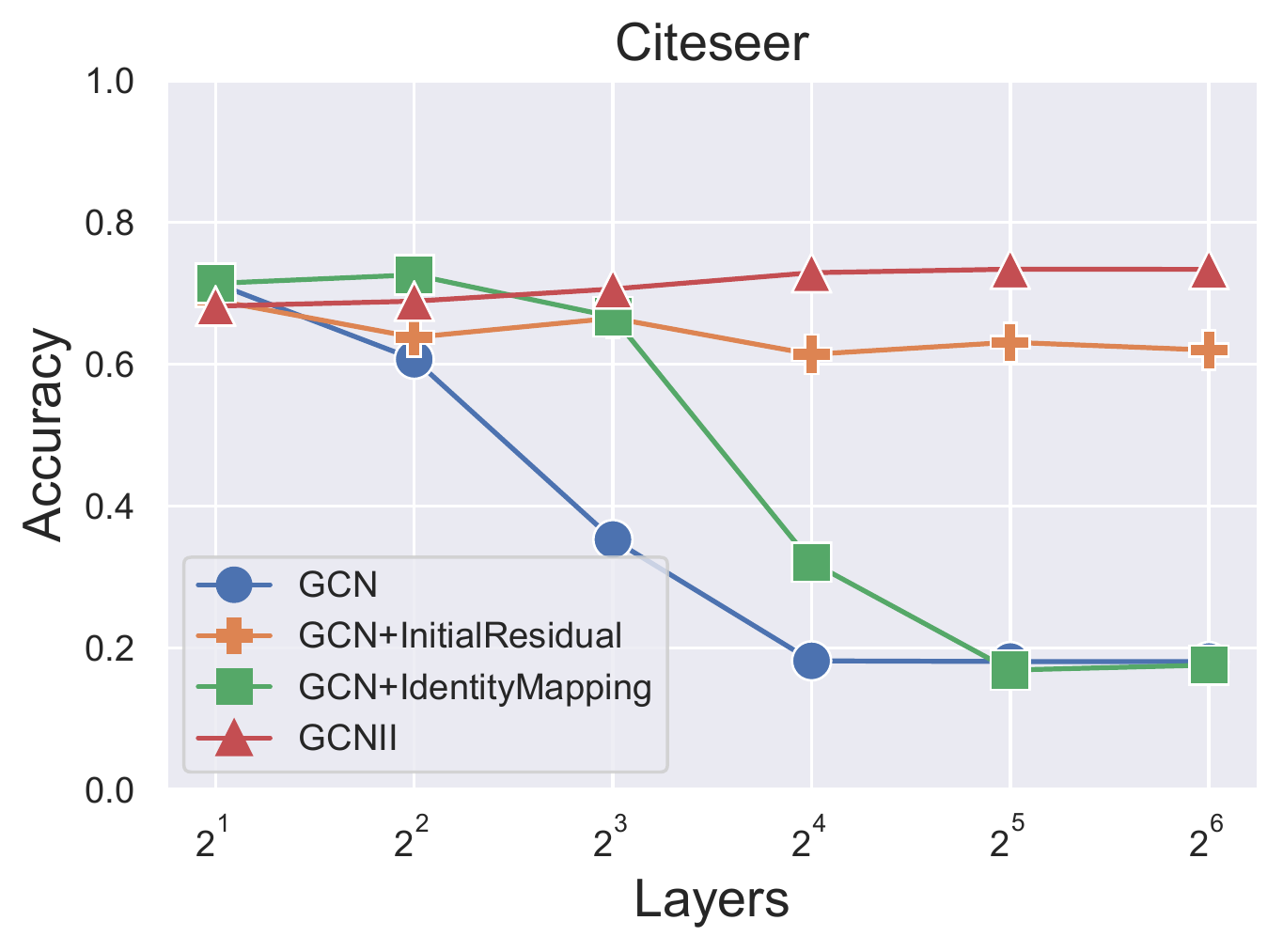}
        \includegraphics[height=42mm]{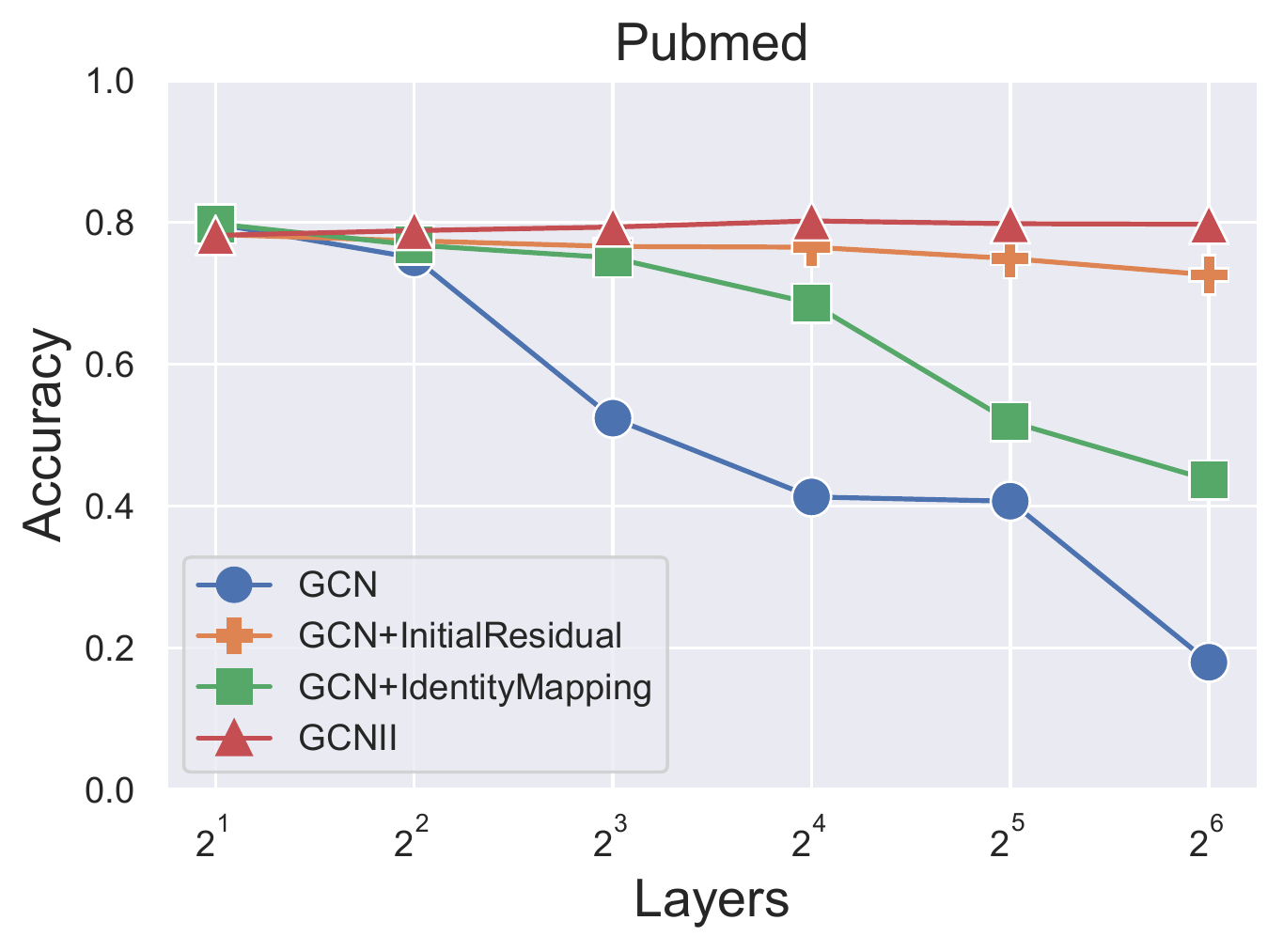}
        \vspace{-2mm}
        \caption{Ablation study on initial residual and identity mapping.} 
        \label{fig:ablation}
            %  \vspace{-3mm}
    % \end{small}
    \end{center}
    \vskip -0.1in
  \end{figure*}

\subsection{Over-Smoothing Analysis for GCN}
Recall that Conjecture~\ref{con:degree} suggests that nodes with higher degrees are
more likely to suffer from over-smoothing. To verify this conjecture,
we study how the classification accuracy varies with node degree in
the semi-supervised node classification task on Cora, Citeseer, and
Pubmed. More specifically, we group the nodes of
each graph according to their degrees. The $i$-th group consists of
nodes with degrees in the range $[2^i, 2^{i+1})$ for $i=0,\ldots,
\infty$. For each group, we report the average classification accuracy
of GCN with residual connection with various network depths in
Figure~\ref{fig:degree_acc}.

We have the following observations.
First of all, we note that the
accuracy of the 2-layer GCN model increases with the node degree. This is as expected, as nodes with
higher degrees generally gain more information from their neighbors. However, as
we extend the network depth, the accuracy of high-degree nodes
drops more rapidly than that of low-degree nodes. Notably, GCN with 64
layers is unable to classify  nodes with degrees larger than 100. 
This suggests that over-smoothing indeed has a greater impact on nodes
with higher degrees.

\subsection{Ablation Study}
Figure~\ref{fig:ablation} shows the results of an ablation study that evaluates the contributions
of our two techniques: initial residual connection and identity
mapping.  We make three observations from Figure~\ref{fig:ablation}:
1) Directly applying identity mapping to the vanilla GCN
retards the effect of over-smoothing marginally. 2) Directly applying initial
residual connection to the vanilla GCN relieves
over-smoothing significantly. However, the best performance is still achieved by the
2-layer model. 3) Applying identity mapping and initial residual
connection simultaneously ensures that  the accuracy increases with the
network depths. This result suggests that both techniques are needed
to solve the problem of over-smoothing.

\section{Conclusion}
\label{sec:conclusion}
We propose GCNII, a simple and deep GCN model that prevents over-smoothing by initial residual connection and identity mapping.  The theoretical analysis shows that GCNII is able to express a $K$ order polynomial filter with arbitrary coefficients.  For vanilla GCN with multiple layers,  we provide theoretical and empirical evidence that nodes with higher degrees are more likely to suffer from over-smoothing. Experiments show that the deep GCNII model achieves new state-of-the-art results on various semi- and full-supervised tasks. Interesting directions for future work include combining GCNII with the attention mechanism and analyzing the behavior of GCNII with the ReLU operation.

\section*{Acknowledgements}
This research was supported in part by National Natural Science Foundation of China (No. 61832017, No. 61932001 and No. 61972401), by Beijing Outstanding Young Scientist Program NO. BJJWZYJH012019100020098, by the Fundamental Research Funds for the Central Universities and the Research Funds of Renmin University of China under Grant 18XNLG21, by Shanghai Science and Technology Commission (Grant No. 17JC1420200), 
by Shanghai Sailing Program (Grant No. 18YF1401200) and a research fund supported by Alibaba Group through Alibaba Innovative Research Program.

\bibliography{reference}
\bibliographystyle{icml2020}

\clearpage
\appendix
\section{Proofs}
\subsection{Proof of Theorem~\ref{thm:GCNII}}
  \begin{proof}
For simplicity, we assume the signal vector $\mathbf{x}$ to be
    non-negative. Note that we can convert $\mathbf{x}$  into a
    non-negative input layer $\mathbf{H^{(0)}}$ by a linear
    transformation. We consider a weaker version of GCNII by fixing
    $\alpha_\ell= 0.5$ and fixing the weight matrix  $(1-\beta_\ell) \mathbf{I}_n + \beta_\ell \mathbf{W}^{(\ell)} 
 $ to be $\gamma_\ell \mathbf{I}_n$, where $\gamma_\ell$ is a
 learnable parameter. We have
\begin{equation*}
    \mathbf{H}^{(l+1)} =
    \sigma\left(\tilde{\mathbf{D}}^{-1 / 2}
      \tilde{\mathbf{A}} \tilde{\mathbf{D}}^{-1 / 2}\left(
        \mathbf{H}^{(\ell)}+ \mathbf{x}\right)\gamma_\ell\mathbf{I}_n\right).
  \end{equation*}
   Since the input feature $\mathbf{x}$ is
   non-negative, we can remove the ReLU operation:
   \begin{align*}
    \mathbf{H}^{(\ell+1)} &= \gamma_\ell
 \tilde{\mathbf{D}}^{-1 / 2}
      \tilde{\mathbf{A}} \tilde{\mathbf{D}}^{-1 / 2}\left(
                         \mathbf{H}^{(\ell)}+ \mathbf{x}\right) \\
     &= \gamma_\ell \left(\left(\mathbf{I}_n  -
       \tilde{\mathbf{L}}\right) \cdot  \left(
        \mathbf{H}^{(\ell)}+\mathbf{x}\right) \right).
  \end{align*}
  Consequently, we can express the final representation as

  \begin{equation}
    \label{eqn:GCNII_filter}
  \mathbf{H}^{(K-1)} = \left(\sum_{\ell=0}^{K-1}\left(\prod_{k=K-\ell-1}^{K-1}\gamma_{k} \right)
                       \left(\mathbf{I}_n  - \tilde{\mathbf{L}}
                       \right)^\ell \right) \mathbf{x}.
                       \end{equation}

                       On the other hand,  a polynomial filter of graph
                       $\tilde{G}$ can be expressed as
                       \begin{align*}& \quad \left(\sum_{k=0}^{K-1}
    \theta_{k} \tilde{\mathbf{L}}^{k}\right) \mathbf{x} = \left(\sum_{i=0}^{k}
    \theta_{k}  \left(\mathbf{I}_n -  \left( \mathbf{I}_n  -
                       \tilde{\mathbf{L}} \right)\right)^{k}\right)
                                                          \mathbf{x}\\
                        & = \left(\sum_{k=0}^{K-1}
    \theta_{k} \left( \sum_{\ell=0}^k(-1)^\ell {k\choose \ell} \left( \mathbf{I}_n  -
                       \tilde{\mathbf{L}} \right)^\ell\right)\right)
                          \mathbf{x}.
                       \end{align*}
                       Switching the order of summation  follows
                       that a $K$-order polynomial fiter $\left(\sum_{k=0}^{K-1}
    \theta_{k} \tilde{\mathbf{L}}^{k}\right) \mathbf{x} $ can be
  expressed as 
  \begin{equation}
    \label{eqn:filter}
  \left(\sum_{k=0}^{K-1}
    \theta_{k} \tilde{\mathbf{L}}^{k}\right) \mathbf{x} \hspace{-1mm}= \hspace{-1mm}  \left(\sum_{\ell=0}^{K-1} \left(\sum_{k=\ell}^{K-1}
    \theta_{k} (-1)^\ell {k\choose \ell}\right) \hspace{-1mm} \left( \mathbf{I}_n  -
                       \tilde{\mathbf{L}} \right)^\ell\right)
                   \mathbf{x}.
                   \end{equation}

                      To  show that GCNII can express an arbitrary  $K$-order polynomial filter,  we need to prove that
                         there exists a solution $\gamma_{\ell}$,
                         $\ell=0,\ldots, K-1$ such that the
                         corresponding coefficients  of $\left( \mathbf{I}_n  -
                       \tilde{\mathbf{L}} \right)^\ell$ in equations~\eqref{eqn:GCNII_filter}
                       and~\eqref{eqn:filter} are equivalent. More
                       precisely, we need to show the
                         following equation system 
                         $$\prod_{k=K-\ell-1}^{K-1}\gamma_{k}
                         = \sum_{k=\ell}^{K-1}
                         \theta_{k} (-1)^\ell {k\choose \ell},\,\, k=0,\ldots, K-1,$$
                         has a solution $\gamma_{\ell}$,
                         $\ell=0,\ldots, K-1$. Since the left-hand
                         side is a partial product of $\gamma_{k}$ from
                         $K-\ell-1$ to $K-1$, we can solve the
                         equation system by 
                         \begin{equation}
                           \label{eqn:theta}
                         \gamma_{K-\ell-1} = \left. \sum_{k=\ell}^{K-1}
                         \theta_{k} (-1)^\ell {k\choose \ell} \middle/ \sum_{k=\ell-1}^{K-1}
                         \theta_{k} (-1)^{\ell-1} {k\choose \ell-1} \right.,
                         \end{equation}
                         for $\ell=1,\dots, K-1$ and $\gamma_{K-1} = \sum_{k=0}^{K-1}
                         \theta_{k}$.
Note that the above solution may fail when $ \sum_{k=\ell-1}^{K-1}
                         \theta_{k} (-1)^{\ell-1} {k\choose \ell-1} =
                         0 $. In this case, we can set
                         $\gamma_{K-\ell-1}$ sufficiently large so that
                         equation~\eqref{eqn:theta} is still a good
                         approximation.   We also note that this case is rare because it implies
                         that the $K$-order filter ignores all
                         features from the $\ell$-hop neighbors.      This
       proves that a $K$-layer GCNII  can express  the $K$-th order polynomial filter $\left(\sum_{i=0}^{k}
    \theta_{i} \mathbf{L}^{i}\right) \mathbf{x}$ with arbitrary
  coefficients $\theta$. 
 \end{proof}

 \subsection{Proof of Theorem~\ref{thm:GCN}}
 To prove Theorem~\ref{thm:GCN}, we need the following {\em Cheeger Inequality}~\cite{chung2007four} for
lazy random walks. 

      \begin{lemma}[\cite{chung2007four}]
        \label{lem:cheeger}
        Let $\mathbf{p}^{(K)}_i = \left({ \mathbf{I}_n+
      \tilde{\mathbf{A}} \tilde{\mathbf{D}}^{-1 } \over 2}\right)^K
      \mathbf{e_i} $ is the $K$-th transition probability vector from
      node $i$ on connected self-looped graph $\tilde{G}$. Let $\lambda_{\tilde{G}}$ denote the
spectral gap of $\tilde{G}$. 
        The $j$-th entry of
      $\mathbf{p}^{(K)}_i$ can be bounded by
      $$ \left|\mathbf{p}^{(K)}_i(j) -   {d_j+1 \over 2m+n} \right| \le \sqrt{d_j+1 \over d_i+1}
\left( 1- {\lambda_{\tilde{G}}^2 \over 2}\right)^K.$$
        \end{lemma}

\begin{proof}[Proof of Theorem~\ref{thm:GCN}]
  Note that $\mathbf{I}_n = \tilde{\mathbf{D}}^{-1 / 2} 
  \tilde{\mathbf{D}}^{1 / 2}$, we have
  \begin{align*}
    \mathbf{h}^{(K)} &= \left({\mathbf{I}_n+\tilde{\mathbf{D}}^{-1 / 2}
      \tilde{\mathbf{A}} \tilde{\mathbf{D}}^{-1 / 2} \over
      2}\right)^K\cdot \mathbf{x}\\
    & = \left(  \tilde{\mathbf{D}}^{-1 / 2} \left({ \mathbf{I}_n+
      \tilde{\mathbf{A}} \tilde{\mathbf{D}}^{-1 } \over 2} \right)
                       \tilde{\mathbf{D}}^{1 / 2}\right)^K\cdot
                       \mathbf{x}\\
    &= \tilde{\mathbf{D}}^{-1 / 2}  \left({ \mathbf{I}_n+
      \tilde{\mathbf{A}} \tilde{\mathbf{D}}^{-1 } \over 2} \right)^K \cdot \left(
                       \tilde{\mathbf{D}}^{1 / 2}
                       \mathbf{x}\right).
  \end{align*}
 We express $\tilde{\mathbf{D}}^{1 / 2}
 \mathbf{x} $ as linear combination of standard basis:
 $$ \tilde{\mathbf{D}}^{1 / 2}
                       \mathbf{x} = \left(\mathbf{D} + \mathbf{I}_n\right)^{1 / 2}
                       \mathbf{x} = \sum_{i=1}^n \left( \mathbf{x}(i)
                         \sqrt{d_i+1} \right)\cdot
                       \mathbf{e_i},$$ it follows that
    \begin{align*}
    \mathbf{h}^{(K)} 
    &= \tilde{\mathbf{D}}^{-1 / 2}  \left({ \mathbf{I}_n+
      \tilde{\mathbf{A}} \tilde{\mathbf{D}}^{-1 } \over 2} \right)^K \cdot  \sum_{i=1}^n \left( \mathbf{x}(i)
                         \sqrt{d_i+1} \right)\cdot
      \mathbf{e_i}\\
      &= \sum_{i=1}^n \mathbf{x}(i) \sqrt{d_i+1} \cdot \tilde{\mathbf{D}}^{-1 / 2}
        \left({ \mathbf{I}_n+
      \tilde{\mathbf{A}} \tilde{\mathbf{D}}^{-1 } \over 2}\right)^K \cdot 
      \mathbf{e_i}.
    \end{align*}

  We note that  $ \left({ \mathbf{I}_n+
      \tilde{\mathbf{A}} \tilde{\mathbf{D}}^{-1 } \over 2}\right)^K \cdot 
      \mathbf{e_i}  = \mathbf{p}^{(K)}_i $ is the $K$-th transition
      probability vector of a random walk from
     node $i$. By Lemma~\ref{lem:cheeger}, the $j$-th entry of
      $\mathbf{p}^{(K)}_i $ can be bounded by
      $$ \left|\mathbf{p}^{(K)}_i(j)   -  {d_j+1 \over 2m+n} \right| \le \sqrt{d_j+1 \over d_i+1}
\left( 1- {\lambda_{\tilde{G}}^2 \over 2}\right)^K,$$
or equivalently, 
$$\mathbf{p}^{(K)}_i(j) = {d_j+1 \over 2m+n} \pm \sqrt{d_j+1 \over d_i+1}
\left( 1- {\lambda_{\tilde{G}}^2 \over 2}\right)^K.$$
Therefore, we can express the $j$-th entry of $ \mathbf{h}^{(K)}$  as  
\begin{align*}
   & \quad \mathbf{h}^{(K)}(j)
= \left(\sum_{i=1}^n\sqrt{d_i+1}\mathbf{x}(i) \cdot \tilde{\mathbf{D}}^{-1 /
  2}\mathbf{p}^{(K)}_i\right)(j) \\
   &\hspace{-2mm} = \hspace{-1mm}\sum_{i=1}^n\sqrt{d_i \hspace{-0.7mm} +\hspace{-0.7mm}  1}\mathbf{x}(i)
  {1\over \sqrt{d_j \hspace{-0.7mm} +\hspace{-0.7mm}  1 }
     \hspace{-1mm}}\cdot  \left( \hspace{-1mm} {d_j \hspace{-0.7mm} +\hspace{-0.7mm}  1 \over 2m \hspace{-0.7mm} +\hspace{-0.7mm}  n} \hspace{-1mm} \pm \hspace{-1mm}\sqrt{d_j \hspace{-0.7mm} +\hspace{-0.7mm}  1 \over d_i \hspace{-0.7mm} +\hspace{-0.7mm}  1}
                          \left( 1 \hspace{-1mm}-\hspace{-1mm} {\lambda_{\tilde{G}}^2 \over 2}\right)^K \hspace{-1mm}\right)\\
      &\hspace{-2mm} = \sum_{i=1}^n{\sqrt{(d_j+1)(d_i+1)} \over 2m+n}\mathbf{x}(i)
        \pm \sum_{i=1}^n \mathbf{x}(i)  \left( 1- {\lambda_{\tilde{G}}^2 \over 2}\right)^K.
    \end{align*}
    This proves
    \begin{equation*}
  \mathbf{h}^{(K)} =  {\left<\tilde{\mathbf{D}}^{1 / 2} \mathbf{1}, \mathbf{x} \right> \over 2m+n}
\tilde{\mathbf{D}}^{1 / 2} \mathbf{1} \pm\left( \sum_{i=1}^n x_i
  \right) \cdot \left( 1-
    {\lambda_{\tilde{G}}^2 \over 2}\right)^K \cdot  \mathbf{1},
\end{equation*}
and the Theorem follows.
    \end{proof}

\section{Hyper-parameters details}
\label{appendix-hyperparameters}

Table~\ref{semi-hyperparameters} summarizes the training configuration of GCNII for semi-supervised. $L_{2_{d}}$ and $L_{2_{c}}$ denote the weight decay for dense layer and convolutional layer respectively. The searching hyper-parameters include numbers of layers, hidden dimension, dropout, $\lambda$ and $L_{2_{c}}$ regularization.

Table~\ref{full-hyperparameters} summarizes the training configuration of all model for full-supervised. We use the full-supervised hyper-parameter setting from DropEdge for JKNet and IncepGCN on citation networks. For other cases, grid search was performed over the following search space: layers (4, 8, 16, 32 ,64), dropedge (0.1, 0.2, 0.3, 0.4, 0.5, 0.6, 0.7, 0.8, 0.9), $\alpha_\ell$ (0.1, 0.2, 0.3, 0.4, 0.5), $\lambda$ (0.5, 1, 1.5), $L_2$ regularization (1e-3, 5e-4, 1e-4, 5e-5, 1e-5, 5e-6, 1e-6). 

\begin{table}[t]
    \caption{The hyper-parameters for Table~\ref{semi-table}.}
    \label{semi-hyperparameters}
    \vskip 0.1in
    \begin{center}
    \begin{small}
        \setlength{\tabcolsep}{1.2mm}{
        \begin{tabular}{l|l}
        \toprule
            Dataset               & Hyper-parameters \\
        \midrule
            \multirow{1}{*}{Cora} & \begin{tabular}[c]{@{}l@{}}layers: 64, $\alpha_\ell$: 0.1, lr: 0.01, hidden: 64, $\lambda$: 0.5,\\ dropout: 0.6, $L_{2_{c}}$: 0.01, $L_{2_{d}}$: 0.0005 \end{tabular} \\
            \midrule
            \multirow{1}{*}{Citeseer} & \begin{tabular}[c]{@{}l@{}}layers: 32, $\alpha_\ell$: 0.1, lr: 0.01, hidden: 256, $\lambda$: 0.6,\\ dropout: 0.7, $L_{2_{c}}$: 0.01, $L_{2_{d}}$: 0.0005 \end{tabular} \\
            \midrule
            \multirow{1}{*}{Pubmed} & \begin{tabular}[c]{@{}l@{}}layers: 16, $\alpha_\ell$: 0.1, lr: 0.01, hidden: 256, $\lambda$: 0.4,\\ dropout: 0.5, $L_{2_{c}}$: 0.0005, $L_{2_{d}}$: 0.0005 \end{tabular} \\
            \bottomrule
            \end{tabular}}
    \end{small}
    \end{center}
    \vskip -0.1in
  \end{table}
  
\begin{table}[t]
    \caption{The hyper-parameters for Table~\ref{fulltrain-table}.}
    \label{full-hyperparameters}
    \vskip 0.1in
    \begin{center}
    \begin{small}
        \setlength{\tabcolsep}{1.2mm}{
        \begin{tabular}{ll|l}
        \toprule
            Dataset               & Method & Hyper-parameters \\
        \midrule
            \multirow{2}{*}{Cora} &\multirow{1}{*}{APPNP} & \begin{tabular}[c]{@{}l@{}}$\alpha$: 0.1, $L_2$: 0.0005, lr: 0.01, hidden: 64, \\dropout: 0.5 \end{tabular} \\ 
                                  &\multirow{1}{*}{GCNII} & \begin{tabular}[c]{@{}l@{}}layers: 64, $\alpha_\ell$: 0.2, lr: 0.01, hidden: 64,\\ $\lambda$: 0.5, dropout: 0.5, $L_2$: 0.0001 \end{tabular} \\
            \midrule
            \multirow{2}{*}{Cite.} &\multirow{1}{*}{APPNP} & \begin{tabular}[c]{@{}l@{}}$\alpha$: 0.5, $L_2$: 0.0005, lr: 0.01, hidden: 64, \\dropout: 0.5 \end{tabular} \\ 
                                   &\multirow{1}{*}{GCNII} & \begin{tabular}[c]{@{}l@{}}layers: 64, $\alpha_\ell$: 0.5, lr: 0.01, hidden: 64,\\ $\lambda$: 0.5, dropout: 0.5, $L_2$: 5e-6 \end{tabular} \\
            \midrule
            \multirow{2}{*}{Pubm.} &\multirow{1}{*}{APPNP} & \begin{tabular}[c]{@{}l@{}}$\alpha$: 0.4, $L_2$: 0.0001, lr: 0.01, hidden: 64, \\dropout: 0.5 \end{tabular} \\ 
                                   &\multirow{1}{*}{GCNII} & \begin{tabular}[c]{@{}l@{}}layers: 64, $\alpha_\ell$: 0.1, lr: 0.01, hidden: 64,\\ $\lambda$: 0.5, dropout: 0.5, $L_2$: 5e-6 \end{tabular} \\
            \midrule
            \multirow{4}{*}{Cham.}  &\multirow{1}{*}{APPNP} & \begin{tabular}[c]{@{}l@{}}$\alpha$: 0.1, $L_2$: 1e-6, lr: 0.01, hidden: 64, \\dropout: 0.5 \end{tabular} \\ 
                                    &\multirow{1}{*}{JKNet} & \begin{tabular}[c]{@{}l@{}}layers: 32, lr: 0.01, hidden: 64,\\ dropedge: 0.7, dropout: 0.5, $L_2$: 0.0001 \end{tabular} \\
                                    &\multirow{1}{*}{IncepGCN} & \begin{tabular}[c]{@{}l@{}}layers: 8, lr: 0.01, hidden: 64,\\ dropedge: 0.9, dropout: 0.5, $L_2$: 0.0005 \end{tabular} \\ 
                                    &\multirow{1}{*}{GCNII} & \begin{tabular}[c]{@{}l@{}}layers: 8, $\alpha_\ell$: 0.2, lr: 0.01, hidden: 64,\\ $\lambda$: 1.5, dropout: 0.5, $L_2$: 0.0005 \end{tabular} \\
            \midrule
            \multirow{4}{*}{Corn.}  &\multirow{1}{*}{APPNP} & \begin{tabular}[c]{@{}l@{}}$\alpha$: 0.5, $L_2$: 0.005, lr: 0.01, hidden: 64, \\dropout: 0.5 \end{tabular} \\
                                    &\multirow{1}{*}{JKNet} & \begin{tabular}[c]{@{}l@{}}layers: 4, lr: 0.01, hidden: 64,\\ dropedge: 0.5, dropout: 0.5, $L_2$: 5e-5 \end{tabular} \\
                                    &\multirow{1}{*}{IncepGCN} & \begin{tabular}[c]{@{}l@{}}layers: 16, lr: 0.01, hidden: 64,\\ dropedge: 0.7, dropout: 0.5, $L_2$: 5e-5 \end{tabular} \\ 
                                    &\multirow{1}{*}{GCNII} & \begin{tabular}[c]{@{}l@{}}layers: 16, $\alpha_\ell$: 0.5, lr: 0.01, hidden: 64,\\ $\lambda$: 1, dropout: 0.5, $L_2$: 0.001 \end{tabular} \\
            \midrule
            \multirow{4}{*}{Texa.}  &\multirow{1}{*}{APPNP} & \begin{tabular}[c]{@{}l@{}}$\alpha$: 0.5, $L_2$: 0.001, lr: 0.01, hidden: 64, \\dropout: 0.5 \end{tabular} \\ 
                                    &\multirow{1}{*}{JKNet} & \begin{tabular}[c]{@{}l@{}}layers: 32, lr: 0.01, hidden: 64,\\ dropedge: 0.8, dropout: 0.5, $L_2$: 5e-5 \end{tabular} \\
                                    &\multirow{1}{*}{IncepGCN} & \begin{tabular}[c]{@{}l@{}}layers: 8, lr: 0.01, hidden: 64,\\ dropedge: 0.8, dropout: 0.5, $L_2$: 5e-6 \end{tabular} \\
                                    &\multirow{1}{*}{GCNII} & \begin{tabular}[c]{@{}l@{}}layers: 32, $\alpha_\ell$: 0.5, lr: 0.01, hidden: 64,\\ $\lambda$: 1.5, dropout: 0.5, $L_2$: 0.0001 \end{tabular} \\
            \midrule
            \multirow{4}{*}{Wisc.}  &\multirow{1}{*}{APPNP} & \begin{tabular}[c]{@{}l@{}}$\alpha$: 0.5, $L_2$: 0.005, lr: 0.01, hidden: 64, \\dropout: 0.5 \end{tabular} \\
                                    &\multirow{1}{*}{JKNet} & \begin{tabular}[c]{@{}l@{}}layers: 8, lr: 0.01, hidden: 64,\\ dropedge: 0.8, dropout: 0.5, $L_2$: 5e-5 \end{tabular} \\
                                    &\multirow{1}{*}{IncepGCN} & \begin{tabular}[c]{@{}l@{}}layers: 8, lr: 0.01, hidden: 64,\\ dropedge: 0.7, dropout: 0.5, $L_2$: 0.0001 \end{tabular} \\
                                    &\multirow{1}{*}{GCNII} & \begin{tabular}[c]{@{}l@{}}layers: 16, $\alpha_\ell$: 0.5, lr: 0.01, hidden: 64,\\ $\lambda$: 1, dropout: 0.5, $L_2$: 0.0005 \end{tabular} \\

            \bottomrule
            \end{tabular}}
    \end{small}
    \end{center}
    \vskip -0.1in
\end{table}
\end{document}